\documentclass[twoside]{article}

\usepackage[accepted]{aistats2021}
\usepackage[round]{natbib}

\usepackage{xr}

\usepackage{amsmath}
\usepackage{amssymb}
\usepackage{amsthm}
\usepackage{mathrsfs}
\usepackage{newtxtext}
\usepackage{docmute}
\usepackage{mathtools}

\usepackage[dvipdfmx]{graphicx}
\usepackage[dvipdfmx]{color}

\usepackage{amscd,amsfonts,braket,mathrsfs,latexsym,enumerate}

\usepackage[pdftex,dvipsnames]{xcolor}

\usepackage[all]{xy}

\usepackage{bm}

\usepackage{here}

\usepackage{url}

\newtheorem{thm}{Theorem}

\newtheorem{prop}{Proposition}

\newtheorem{ex}{Example}
\newtheorem{df}{Definition}

\theoremstyle{remark}

\newcommand{\RR}{\mathbb{R}} 

\newcommand{\ul}[1]{\underline{#1}} 
\newcommand{\ol}[1]{\overline{#1}} 
\newcommand{\wt}[1]{\widetilde{#1}} 
\newcommand{\wh}[1]{\widehat{#1}} 

\newcommand{\Ch}{{\rm Ch}}

\newcommand{\ReLU}{\mathrm{ReLU}\ \! } 

\newcommand{\ca}[1]{\mathcal{#1}} 

\def\inprod<#1>{\left\langle #1 \right\rangle} 

\mathtoolsset{showonlyrefs=true}

\theoremstyle{definition}

\numberwithin{equation}{section}

\begin{document}

%
\runningtitle{On the Number of Linear Functions Composing Deep Neural Network}

%

\twocolumn[
\aistatstitle{On the Number of Linear Functions Composing Deep Neural Network: \\ 
Towards a Refined Definition of Neural Networks Complexity}

\aistatsauthor{Yuuki Takai \And Akiyoshi Sannai \And  Matthieu Cordonnier }

\aistatsaddress{
RIKEN AIP \\ 
\url{yuuki.takai@riken.jp} \And RIKEN AIP \\ 
\url{akiyoshi.sannai@riken.jp} \And \'{E}cole Normale \\ Sup\'{e}rieure Paris-Saclay \\ 
 \url{matthieu.cordonnier}\\ \url{@ens-paris-saclay.fr} 
} ]

\begin{abstract}

The classical approach to measure the expressive power of deep neural networks with piecewise linear activations is based on counting their maximum number of linear regions.
This complexity measure is quite relevant to understand general properties of the expressivity of neural networks such as the benefit of depth over width. Nevertheless, it appears limited when it comes to comparing the expressivity of different network architectures. This lack becomes particularly prominent when considering permutation-invariant networks, 
due to the symmetrical redundancy among the linear regions.
To tackle this, we propose a refined definition of piecewise linear function complexity: instead of counting the number of linear regions directly, we first introduce an equivalence relation among the linear functions composing a piecewise linear function and then count those linear functions relative to that equivalence relation. 
Our new complexity measure can clearly distinguish between the two aforementioned models, is consistent with the classical measure, and increases exponentially with depth.
\end{abstract}

\section{Introduction}

Deep neural networks with rectified linear units (ReLU) as an activation function have been remarkably successful in computer vision, speech recognition, and other domains \citep{krizhevsky2012imagenet}, \citep{goodfellow2013multi}, \citep{wan2013regularization}, 
 \citep{silver2017mastering}. 
However, the theoretical understanding to support this experimental progress is still insufficient, thereby motivating several researchers to bridge this crucial gap.

A fundamental theoretical problem is {\em the expressivity of neural networks}: given an architecture configuration (depth, width, layer type, activation function), which class of functions can neural networks compute and with what level of performance? 
To evaluate the expressive power of a neural network, we, therefore, need to define {\em a measure of its complexity}. 
In the case where ReLU is the only considered activation function, a neural network represents a piecewise linear function, thus, a natural method to measure complexity is to count the number of linear regions. 
 From this perspective, we can theoretically justify 
favoring depth over width. Moreover, this has also been widely established through experimentation  
\citep{pascanu2013number}, \citep{montufar2014number}, \citep{telgarsky2016benefits}, \citep{arora2016understanding}, \citep{eldan2016power}, \citep{yarotsky2017error}, \citep{serra2018bounding}, \citep{chatziafratis2019depth}, \citep{chatziafratis2020better}. 

Nevertheless, 
 as we subsequently show, using the number of linear regions as a straightforward measure, does not adequately reflect the properties of the underlying function which the network represents in some case. 

Concretely, we consider permutation-invariant functions and the model introduced by \citet{zaheer2017deep}. This model has been proven to be a universal approximator for the class of permutation-invariant continuous functions \citep{maron2019universality}, \citep{zaheer2017deep}. 
Since this model is permutation-invariant, its expressive power is strictly lower than that of the fully connected model. However, we point out that the maximal number of linear regions for both the models 
 is asymptotically similar.

This highlights the fact that the straightforward relationship between \textit{number of linear regions} and \textit{expressive power} needs to be qualified, because it cannot distinguish between these two models clearly. Thus, we propose a new complexity measure that enables us to reliably make this distinction.



Our main contribution is to 
introduce such a measure (Definition \ref{def of inv}) and to prove that the invariant model and the fully connected model actually have different values (Theorem \ref{comp full}, Theorem \ref{thm:complexity-invariant-shallow}). To define our measure, we consider not the \textit{number of linear regions} but  the \textit{number of linear functions on them}. Our measure counts them relative to a certain equivalence relation. This relation identifies linear functions (and their inherent linear region) that can be mapped from one to another through a certain Euclidean transformation, i.e., isometric affine transformation. 
As a more intrinsic example, we consider digital images. 
The complexity of digital images is usually estimated by the number of pixels and the number of colors. 
The number of pixels corresponds to the conventional measure of complexity, i.e., the number of linear regions $c^{\#}$, 
and the number of colors corresponds to our proposed measure of 
complexity, i.e., the number of linear functions $c^\sim$. 
From this point of view, our measure is a natural consequence.  
We remark that other possible measures of complexity such as using 
Betti numbers of the linear regions \citep{bianchini2014complexity}, trajectories in the input space \citep{raghu2017expressive}, or the volumes of the boundaries of linear regions \citep{hanin2019complexity} have been proposed. 

The low expressive power of permutation-invariant shallow networks   
is translated to the fewness of linear functions ($\approx$``colors'') composing 
 them, due to permutation invariance. 
Indeed, we show that for permutation-invariant shallow models, the proposed measure of complexity is the same as the number of orbits of linear regions by permutation action and that this number 
is relatively small. Our demonstration relies on theory of hyperplane 
 arrangement which is stable by group action studied
 in \citep{kamiya2012arrangements}. 
 In Section~\ref{sec:deeper-models}, we modified the argument of 
 \citet{montufar2014number} to prove the benefit of depth over width. Particularly, the complexity of the proposed method increases exponentially 
 with depth for fully connected deep models  
 as well as for permutation-invariant deep models introduced by  \citet{zaheer2017deep}.

\section{Preliminaries and background}  

%

A (feedforward) neural network of depth $L+1$ is a composition of 
layers of units which defines a function 
$F\colon \RR^{n_0} \to \RR^{n_{L+1}}$ of 
the form
\begin{align}
 F(\bm{x}) = f_{L+1} \circ g_{L} \circ f_{L} \circ \cdots \circ 
 g_1 \circ f_{1} (\bm{x}), \label{eq:deep-fully-connected}
\end{align}
where $f_l \colon \RR^{n_{l-1}} \to \RR^{n_l}$ is an 
affine map and $g_l \colon \RR^{n_l} \to \RR^{n_l}$ 
is a nonlinear activation function. 
Throughout this paper, as activation functions, we consider the rectifier linear units (ReLU), i.e.,  
for $\bm{x} = (x_1, \dots, x_{n_l})^\top  \in \RR^{n_l}$, 
\[
 \ReLU(\bm{x}) = ( \max\{0, x_1 \}, \dots,  \max\{0, x_{n_l} \})^\top 
  \in \RR^{n_l}. 
\]
Let $\ul{n} = (n_0, n_1, \dots, n_{L+1})$ and $K$ be a connected compact $n_0$-dimensional subset of $\RR^{n_0}$. 
Then, we define $\ca{H}_K^{\text{full}}(\ul{n}) = \ca{H}_K^{\text{full}}(n_0, n_1, \dots, n_{L+1})$ 
as the set of the restriction to $K$ of the neural networks of the form \eqref{eq:deep-fully-connected} with 
$g_l = \ReLU$ for any $l = 1, \dots, L$. We call such a 
network a ReLU neural network. 
The affine transformation $f_l$ can be written as 
$f_l(\bm{x}) = W_l \bm{x} + \bm{c}_l$ with a weight matrix 
$W_l \in \RR^{n_l \times n_{l-1}}$ and a bias vector 
$\bm{c}_l \in \RR^{n_l}$. 
We call the feedforward neural network  {\em shallow} (resp. {\em deep}) 
if $L = 1$ (resp. $L>1$). 

Because $\ReLU$ is a continuous piecewise linear function, 
a function realized by a ReLU neural network is also 
continuous and piecewise linear. 
We are interested in the structures of 
such piecewise linear functions. 
Any piecewise linear function is encoded as the set of pairs made of 
a linear region and a linear function on it.  
Here, for a connected compact $m$-dimensional subset $K$ in $\RR^m$ and a piecewise linear function $f \colon K \to \RR^n$, 
a connected region $D \subset K$ is called {\em a linear region of} $f$ 
if $f$ is linear on $D$ and for any connected region 
$D' \subset K$ satisfying $D \subsetneq D'$, $f$ is not 
linear on $D'$. 
For a piecewise linear function $f$, $c^\#(f)$ denotes the number of 
linear regions of $f$. 
For a set of piecewise linear functions $\ca{H}$, we set 
$c^\#(\ca{H}) = \max\{ c^\#(f) \mid f\in \ca{H} \}$.




%
%
%

\subsection{The number of linear regions for shallow fully connected neural networks}\label{sec:num-of-region-shallow-full}


To calculate the maximum number of linear regions for shallow ReLU neural networks, we use arguments from hyperplane arrangement theory as in \citep{pascanu2013number}.
Let us consider a shallow ReLU neural network 
$F \in \ca{H}_K^{\text{full}}(n_0, n_1, n_2)$, i.e., a network of the form 
\begin{align}
 F(\bm{x}) = f_2 \circ g_1 \circ f_1( \bm{x}), \label{eq:shallow-fully-connected-model}
\end{align}
where $f_1 \colon K \to \RR^{n_1}$ and 
$f_2 \colon \RR^{n_1} \to \RR^{ n_2}$ are two affine maps and 
$g_1 \colon \RR^{n_1} \to \RR^{n_1}$ is ReLU. 

The linear regions of $F$ depend only on the affine map $f_1$. 
We write $f_1(\bm{x}) = W \bm{x} + \bm{c}$ for 
$W = (a_{ij}) \in \RR^{n_1\times n_0}$ and $\bm{c} = (c_i) \in \RR^{n_1}$. 
Let $H_i$ be the hyperplane in $\RR^{n_0}$ defined as 
\[
 a_{i1} x_1 + \cdots + a_{in_0} x_{n_0} + c_i = 0 \ \cdots \ H_i
\]
for $i = 1, \dots, n_1$. 
Then, the linear regions of $F$ are exactly the chambers of the hyperplanes 
arrangement defined by $\ca{A} = \{ H_1, \dots, H_{n_1} \}$, 
i.e., the connected components of the complement 
$\RR^{n_0}{\setminus}\bigcup_{i} H_i$. Let $\Ch(\ca{A})$ denotes the 
set of chambers of arrangement $\ca{A}$. 
Then, Schl\"{a}fli showed that the cardinality 
$|\Ch(\ca{A})|$ of $\Ch(\ca{A})$ satisfies  
\begin{align}
 | \Ch(\ca{A})| \leq 
  \sum_{i = 0}^{n_0} { n_1 \choose i } \label{eq:trivial-upperbound}
\end{align}
and the equality holds if $\ca{A}$ is in general position \citep[Introduction]{orlik2013arrangements}. 
Here, we say that 
the hyperplane arrangement $\ca{A} = \{ H_1, \dots, H_{n_1} \}$ 
is {\em in general position} if $\ca{A}$ satisfies that 
for any $r = 1, \dots, n_0$, 
the codimension of 
the intersection $H_{i_1}\cap \cdots \cap H_{i_r}$ 
is equal to $r$ if $r \leq n_1$ and 
$H_{i_1}\cap \cdots \cap H_{i_r} = \emptyset$ if $r>n_1$ (see Appendix~\ref{subsec:Fully-connected-shallow-example} for an illustration).  
For the hyperplane arrangement $\ca{A}$ defined by the fully connected shallow neural network above, we remark that it is always possible to make it being in general position by perturbing the weight matrix $W$ and the bias vector $\bm{c}$. 
Moreover, for any connected compact $n_0$-dimensional 
subset $K\subset \RR^{n_0}$ and a hyperplane arrangement $\ca{A}$, 
we can take another hyperplane arrangement $\ca{A}'=\{ H_i' \mid i=1, \dots, n_1\}$ such that $|\Ch(\ca{A})|$ is equal to the number of 
connected components of $K{\setminus}\bigcup_{i}(K\cap H_i')$ 
by translating or scaling $\ca{A}$ if it is necessary. 
In particular, the maximal number $c^\#(\ca{H}_K^{\mathrm{full}}(n_0, n_1,n_2))$ of linear regions of the fully connected shallow ReLU neural network having a $n_0$-dimensional input layer and a $n_1$-dimensional hidden layer is $\sum_{i = 0}^{n_0} { n_1 \choose i }$.
For $n_0$ such that $0 \leq n_0\leq n_1/2$, 
by \citep[Section~4.7]{ash1965information}, 
the estimate of the sum 
of binomial coefficients is 
\begin{align}
 \frac{2^{n_1H(n_0/n_1)}}{\sqrt{8n_0(1-n_0/n_1)}} &\leq 
 {n_1 \choose n_0}  
 \leq c^\#(\ca{H}_K^{\mathrm{full}}(n_0, n_1,n_2)) \notag \\ & = 
 \sum_{i = 0}^{n_0} { n_1 \choose i} \leq 2^{n_1H(n_0/n_1)},  
 \label{eq:estimate-linear-region-fully-connected}
\end{align}
where $H(p)$ is the binary entropy function defined as 
\[
 H(p)=-p\log_{2}p-(1-p)\log_{2}(1-p) 
\]
for $0< p <1 $ and $H(0) = H(1)=0$. 

\subsection{The number of linear regions for the permutation invariant 
model} 
\label{sec:num-of-region-shallow-invariant} 

We review the permutation invariant shallow model introduced in \citep{zaheer2017deep} and show that this model can have as many linear regions as a fully connected shallow neural network, 
though this model has a lower expressive power than the fully connected model. We illustrate this calculation on a simple example in Appendix~\ref{subsec:permutation-invariant-shallow-example}.


We define the permutation action on $(\RR^n)^m$ of permutation group 
$S_n$ by the following way. For $\sigma \in S_n$ and 
$\bm{X} = (\bm{x}_1, \dots, \bm{x}_m) \in (\RR^n)^m$ where 
$\bm{x}_i = (x_{i1}, \dots, x_{in})^\top \in \RR^n$, we define
\begin{align*}
 \sigma\cdot \bm{X} &= (\sigma\cdot \bm{x}_1, \dots, 
 \sigma\cdot \bm{x}_n), \\ \sigma\cdot \bm{x}_i &= (x_{i\sigma^{-1}(1)}, \dots, x_{i \sigma^{-1}(n)})^\top.  
\end{align*}
For a subset $K$ of $\RR^n$, we say that $K$ is stable by permutation action if for any $\bm{x} \in K$, $\sigma\cdot \bm{x}\in K$ holds 
for any $\sigma \in S_n$. 

We consider a permutation invariant shallow network as 
\begin{align}
 F(\bm{x}) = f_2 \circ g_1 \circ f_1(\bm{x}) \label{eq:shallow-permutation-invariant}
\end{align}
where $f_1 \colon \RR^n \to (\RR^n)^m$ is a permutation 
equivariant affine map, i.e., 
$f_1(\sigma\cdot \bm{x}) = \sigma\cdot f_1(\bm{x})$ for any 
$\sigma \in S_n$ and $\bm{x}\in \RR^n$, 
and $f_2\colon (\RR^n)^m \to \RR^{m'}$ is a permutation invariant 
affine map i.e., 
$f_2(\sigma\cdot \bm{X}) = f_2(\bm{X})$ for any 
$\bm{X}\in (\RR^n)^m$ and $\sigma \in S_n$, and 
$g_1\colon (\RR^n)^m \to (\RR^n)^m$ is ReLU.    
Then, the realized function $F$ is permutation invariant, i.e., 
$F(\sigma\cdot \bm{x}) = F(\bm{x})$ for any $\sigma\in S_n$. 
Let $K$ be a connected compact $n$-dimensional subset of $\RR^n$ which is stable by permutation action. 
Then, we define $\ca{H}_K^\text{inv}(n, mn, m')$ as the set of 
the restrictions to $K$ of the permutation invariant ReLU neural networks of the form \eqref{eq:shallow-permutation-invariant}. 
By universal approximation theorem \citep{maron2019universality}, 
any permutation-invariant, 
continuous function on $K$ can be approximated by such a 
neural networks for large enough $m'$. 

The set of linear regions of the model depends 
only on the affine map $f_1$ as in the fully connected case. 
Using \citep[Lemma~3]{zaheer2017deep}, 
by the permutation equivariance of $f_1$, if we set 
$f_1(\bm{x}) = W\bm{x} + \bm{c}$ for some 
$W\in (\RR^{n \times n})^m$ and $\bm{c}\in (\RR^{n})^m$, 
these $W$ and $\bm{c}$ can be written as 
\begin{align}
W = \begin{pmatrix}
 	a_1I + b_1(I-\bm{1}\bm{1}^\top) \\ \vdots \\ 
 	a_mI + b_m(I-\bm{1}\bm{1}^\top)
 \end{pmatrix}, \ 
 \bm{c} = \begin{pmatrix}
 	c_1 \bm{1} \\ \vdots \\ c_m\bm{1}
 \end{pmatrix}  \label{eq:equivariant-map}
\end{align}
for some $a_1, \dots, a_m, b_1, \dots, b_m, c_1, \dots, c_m \in \RR$.  
Here, $I$ is the identity matrix in $\RR^{n\times n}$ and $\bm{1}$ is 
the all one vector in $\RR^n$. 
Thus, the set of linear regions of $F$ is equal to the set of chambers of 
the hyperplanes arrangement $\ca{B}_{m,n} = \{H_{11}, \dots, H_{mn} \}$ 
defined for $i = 1, \dots, m$ by, 
\begin{align}
\begin{cases}
	a_i x_1 + b_i x_2 + \cdots + b_i x_n  + c_i = 0 & \ \cdots \ H_{i1}\\ 
 	 b_i x_1 + a_i x_2 + \cdots + b_i x_n  + c_i = 0 & \ \cdots \ H_{i2}\\ 
 	 \hspace{80pt} \vdots & \hspace{20pt}  \\ 
 	 b_i x_1 + b_i x_2 + \cdots + a_i x_n  + c_i = 0 & \ \cdots \ H_{in}. 
\end{cases}	
\label{eq:invariant-equations}
\end{align}


We calculate the number of chambers of the arrangement $\ca{B}_{m,n}$. 
As in inequation \eqref{eq:trivial-upperbound}, the number of chambers are bounded from above 
by $\sum_{i = 0}^{n} { mn \choose i }$ 
and attains this bound if 
the arrangement $\ca{B}_{m,n}$ is in general position. 
However, this arrangement $\ca{B}_{m,n}$ cannot be in general position. 
Indeed, the hyperplanes in the arrangement $\ca{B}_{m,n}$ satisfy
\begin{align}
H_{i_1,j}\cap H_{i_2,j}\cap H_{i_3,j} &= \emptyset, \label{eq:condition-intersection-1}\\ 
H_{i_1,j_1}\cap H_{i_1,j_2} \cap H_{i_2,j_1} &= 
 H_{i_1,j_1}\cap H_{i_1,j_2} \cap H_{i_2,j_2} \notag \\ 
 &= H_{i_1,j_1}\cap H_{i_2,j_1} \cap H_{i_2,j_2} \notag \\
&  \label{eq:condition-intersection-2}
\end{align}
for $i_1, i_2, i_3 = 1, \dots, m$ and $j, j_1, j_2 = 1, \dots, n$.
Nevertheless, we can calculate the number of chambers of the arrangement $\ca{B}_{m,n}$ by applying the  
Deletion-Restriction theorem 
(Theorem~\ref{thm:deletion-restriction} in Appendix~\ref{sec:proof-num-of-chambers-invariant}) \citep[Theorem 2.56 and Theorem 2.68]{orlik2013arrangements} under the assumption \eqref{eq:condition-intersection-1} and \eqref{eq:condition-intersection-2}. 
The detail of the calculation is in Appendix~\ref{sec:proof-num-of-chambers-invariant}. 
\begin{prop}\label{prop:num-of-chambers-invariant}
We assume that $m> n/2$. 
Then, the maximum $b_{m,n}$ of the number of 
 chambers of $\ca{B}_{m,n}$ is bounded from below by a function 
  $g(m,n)$ which is  polynomial with respect to $m$ of degree $n$, and 
  which the coefficient of the leading term is bounded from 
  below by $(2^{5/4})^n/(n\sqrt{2})$. 
\end{prop}

\subsection{Comparison of the numbers of linear regions} 

To equalize the number of hidden units in both models, 
we consider the fully connected model \eqref{eq:shallow-fully-connected-model} withd $n_0 = n$ and $n_1 = mn$. Let $K$ be a connected compact $n$-dimensional subset of $\RR^n$ which is 
stable by permutation action.   
By a universal approximation theorem 
\citep{sonoda2017neural}, if we 
increase the number of hidden units of the fully connected 
shallow models, 
the elements of $\ca{H}_K^{\text{full}}(n, mn, m')$ can approximate any continuous 
maps on a compact set of $\RR^n$. 
On the other hand, although 
elements of $\ca{H}_K^\text{inv}(n,mn,m')$ are 
also universal approximators for permutation-invariant functions 
 \citep{maron2019universality}, any function which is not permutation-invariant cannot be approximated by the elements of 
 $\ca{H}_K^\text{inv}(n, mn, m')$. 
This implies that the expressive power of the permutation-invariant 
shallow models is strictly lower than for the fully 
connected shallow models. 

In keeping with this observation, 
we compare maximum number of linear regions for the 
fully connected \eqref{eq:shallow-fully-connected-model} and the permutation invariant shallow models \eqref{eq:shallow-permutation-invariant}. 
By the estimate \eqref{eq:estimate-linear-region-fully-connected} with $n_0 = n$ and $n_1 =mn$, we have 
\begin{align*}
 c^\#(\ca{H}_K^{\mathrm{full}})(n, mn, n') &\geq 
 \frac{2^{mnH(1/m)}}{\sqrt{8n(1-1/m)}} \\ 
 &\geq 
 \frac{e^n}{2\sqrt{2n}} m^n + O(m^{n-1}). 
\end{align*}
On the other hand, by 
Proposition~\ref{prop:num-of-chambers-invariant}, 
the maximal number of linear regions of permutation invariant shallow 
models is also bounded from below as  
\[
 c^\#(\ca{H}_K^{\mathrm{inv}})(n, mn, n') \geq 
 \frac{(2^{5/4})^n}{n\sqrt{2}}m^n + O(m^{n-1}). 
\]

In particular, although there is a difference of bases, 
$c^\#(\ca{H}_K^\mathrm{inv})(n,mn,n')$ does also increase 
exponentially with respect to $n$. 
This means that the maximum numbers of 
linear regions cannot represent the 
difference of expressive powers of these models clearly.


This observation indicates that we should consider some refined 
measure for complexity and expressive power to be able 
to distinguish between these two classes of models more clearly.

\section{Measure of complexity as the numbers of equivalent classes of linear functions}  
\label{sec:our-measure}


In this section, we introduce a measure of 
complexity which can distinguish 
permutation-invariant shallow models from fully connected 
shallow models. 
Before proposing a refined measure of complexity, 
we observe the structure of 
piecewise linear functions that are permutation-invariant. 

Let $K$ be a connected compact $n$-dimensional subset of $\RR^n$ which is stable by permutation action and $f \colon K \to \RR^{n'}$ be a piecewise linear function 
which is permutation invariant by the permutation group $S_n$, and $\ca{F}(f) = \{(f_\lambda, D_\lambda) \mid \lambda \in \Lambda \}$ 
the set of pairs of linear regions $D_\lambda \subset K$ of $f$ 
and the linear associated function $f_\lambda$ on $D_\lambda$, i.e., 
$f_\lambda$ is the restriction $f|_{D_\lambda}$ of $f$ on $D_\lambda$. 
We call this set $\ca{F}(f)$ {\em the set of linear functions of} $f$. 
We often abbreviate an element $(f_\lambda, D_\lambda)\in \ca{F}(f)$ to 
$f_\lambda$.  
Then, it is easy to show that 
for any permutation $\sigma \in S_n$ and any linear region $D$ of $f$, 
the image $\sigma(D)$ of $D$ by $\sigma$ is also a linear region. 
By this fact and the permutation invariance of $f$, 
for any $(f_\lambda, D_\lambda)$ and $\sigma\in S_n$, 
there is a $\lambda'$ such that $\sigma(D_\lambda) = D_{\lambda'}$ and 
$f_{\lambda} = f_{\lambda'}\circ \sigma|_{D_{\lambda}}$. 
Here, we regard the permutation $\sigma$ as a linear transformation 
on $\RR^n$. Then, the linear transformation induced 
by permutation $\sigma$ is isometric with respect to $L^2$-norm, 
because the map taking $L^2$-norm $\bm{x} \mapsto \|\bm{x} \|_2$ is permutation-invariant. 


Inspired by this observation, we define an equivalence relation 
$\sim$ on the set of pairs $\ca{F}(f)$ of linear functions and regions for piecewise linear function 
$f \colon K \to \RR^{n'}$ as follows:  

\begin{df}\label{def:complexity}
Let $f \colon K \to \RR^{n'}$ be a piecewise linear function and 
 $\mathcal{F}(f)=\{(f_\lambda, D_\lambda) \mid \lambda \in \Lambda \}$ the set of the linear functions of $f$. 
Then, we say that $f_\lambda$ is equivalent to $f_{\lambda'}$, 
denoted by $f_\lambda \sim f_{\lambda'}$, if there is a 
Euclidean transformation $\phi \colon \RR^n \to \RR^n$ satisfying 
$ {\rm (1)} \ \phi(D_\lambda) = D_{\lambda'}$ and  
${\rm (2)} \ f_\lambda  = f_{\lambda'}\circ \phi|_{D_\lambda}$.
Here, a Euclidean transformation $\phi$ is an affine map written as 
$\phi(\bm{x}) = A\bm{x} + \bm{b}$ for an orthogonal matrix 
$A$ and a vector $\bm{b}$.
\end{df}


We can characterize the invariant function 
for a group action as follows: 
(the proof is in Appendix~\ref{sec:proof-characterization-group-invariant}): 
\begin{prop}\label{prop:characterization-group-invariant}
Let $f$ be a piecewise linear function on $K$ and 
$\ca{F}(f) = \{ (f_\lambda, D_\lambda) \mid \lambda \in \Lambda \}$ the set of linear functions of $f$. 
We assume that there is a set $\Phi=\{\phi_1, \dots, \phi_t\}$ of Euclidean transformations on $\RR^n$ such that for any 
$\phi \in \Phi$ and any linear regions $D_\lambda$, there is a $\lambda'\in\Lambda$ such that
$ {\rm (1)} \ \phi(D_\lambda) = D_{\lambda'}, \ \ \ \ 
{\rm (2)} \ f_\lambda  = f_{\lambda'}\circ \phi|_{D_\lambda}$. 
 Then, $f$ is $\hat{\Phi}$-invariant, where 
$\hat{\Phi}=\langle \phi_1, \dots, \phi_t \rangle$ is the group 
generated by $\Phi$. 
\end{prop}

The relation $\sim$ is an 
 equivalence relation. Then, we propose the following measure of complexity:  
 \begin{df}\label{def of inv}
We define {\em the measure of complexity $c^{\sim}(f)$ of $f$} 
by the number of equivalent classes $\mathcal{F}(f)/{\sim}$. 
For a set $\ca{H}$ of piecewise linear functions, we define {\em the 
measure of complexity $c^{\sim}(\ca{H})$ of $\ca{H}$} by the maximum of $c^{\sim}(f)$ for any 
$f \in \ca{H}$.  
\end{df}
 As a trivial upper bound, 
$c^{\sim}(f)$ is bounded from above by $|\mathcal{F}(f)|$, i.e., 
the number of linear regions. More generally, 
the set $\ca{F}(f)$ of linear functions of $f$ may be infinite. 
However, if $f$ is realized by a ReLU neural network of finite width 
and finite depth, $\ca{F}(f)$ is finite. 

We calculate this measure of complexity for the previous two 
classes of models. 
We remark that any Euclidean transformation 
$\phi$ does not change the volumes of linear regions.  
Thus, if the volumes of two linear regions $D_\lambda$ and $D_{\lambda'}$ are different, then $f_\lambda$ and $f_{\lambda'}$ 
 cannot be equivalent.  
We use this observation later to count the number of equivalent classes. 

\subsection{Examples of our measure of complexity in $1$-dimensional case}
\label{subsec:examples-1D-body}

Here, we show some examples in the 1-dimensional case and calculate 
our measures of complexity. 
For simplicity, we consider them on the interval $K=[0,1]$. 

Let $f$ be a piecewise linear function on $[0,1]$ and 
$[0,1] = \bigcup_{i=1}^m D_i$ be the decomposition by 
linear regions $D_i$ for $f$ and $f_i = f|_{D_i}$. 
Then, $(f_i, D_i) \sim (f_j, D_j)$ holds only if $|D_i| = |D_j|$, 
where $|D|$ for interval $D =[p,q]$ is the length $q-p$. 
Indeed, by Definition~1, there is a Euclidean transformation 
$\phi\colon \RR^1 \to \RR^1; x \mapsto ax + b$ 
such that $\phi(D_j) = D_i$. As $\phi$ is Euclidean transformation, 
$a$ is equal to $\pm 1$. 
If we set $D_i = [p_i, q_i]$, then 
$\phi(D_j) = D_i$ is equivalent to 
$ [p_j, q_j]=[\phi(p_i), \phi(q_i)] =[ap_i + b, aq_i+b]$. 
As $a = \pm 1$, this implies that $|D_j| = |D_i|$. 
Moreover, then, $ap_i + b = p_j$ holds. In particular, 
$b = p_j - ap_i$ holds. Because $a$ is $1$ or $-1$, there are only 
two choices of the Euclidean transformation $\phi\colon D_i \to D_j$. 

Furthermore, we set 
$f_i(x) = \alpha_i x + \beta_i$ for $i=1, \dots, m$.  
Then, $(f_j \circ \phi)|_{D_i} = f_i$ holds. Hence, 
for $x \in D_i$, 
\[
\alpha_i x + \beta_i = \alpha_j (ax + b) + \beta_j = a\alpha_j x + 
\alpha_j b + \beta_j
\]
holds. Thus, we have $\alpha_i = a \alpha_j$ and 
$\beta_i = b \alpha_j + \beta_j$. 

By combining these arguments, there are at most two linear functions  
on $D_j = [p_j, q_j]$ equivalent to $(f_i,D_i)$ where $f_i(x) = \alpha_ix+\beta_i$ and $D_i=[p_i, q_i]$: For $f_i(x) = \alpha_i x+\beta_i$, 
\begin{align*}
 f_j(x) = \begin{cases}  
 \alpha_ix + \beta_i - (p_j-p_i)\alpha_i &  \text{if } a = 1, \\  
 -\alpha_ix + \beta_i + (p_j+p_i)\alpha_i & \text{if } a = -1.
\end{cases}
\end{align*}

Based on this observation, we show three examples on $[0,1]$. 


\begin{figure*}[t]
    \begin{tabular}{ccc}
      \begin{minipage}[t]{0.3\hsize}
    \centering
    \includegraphics[keepaspectratio,scale=0.35]{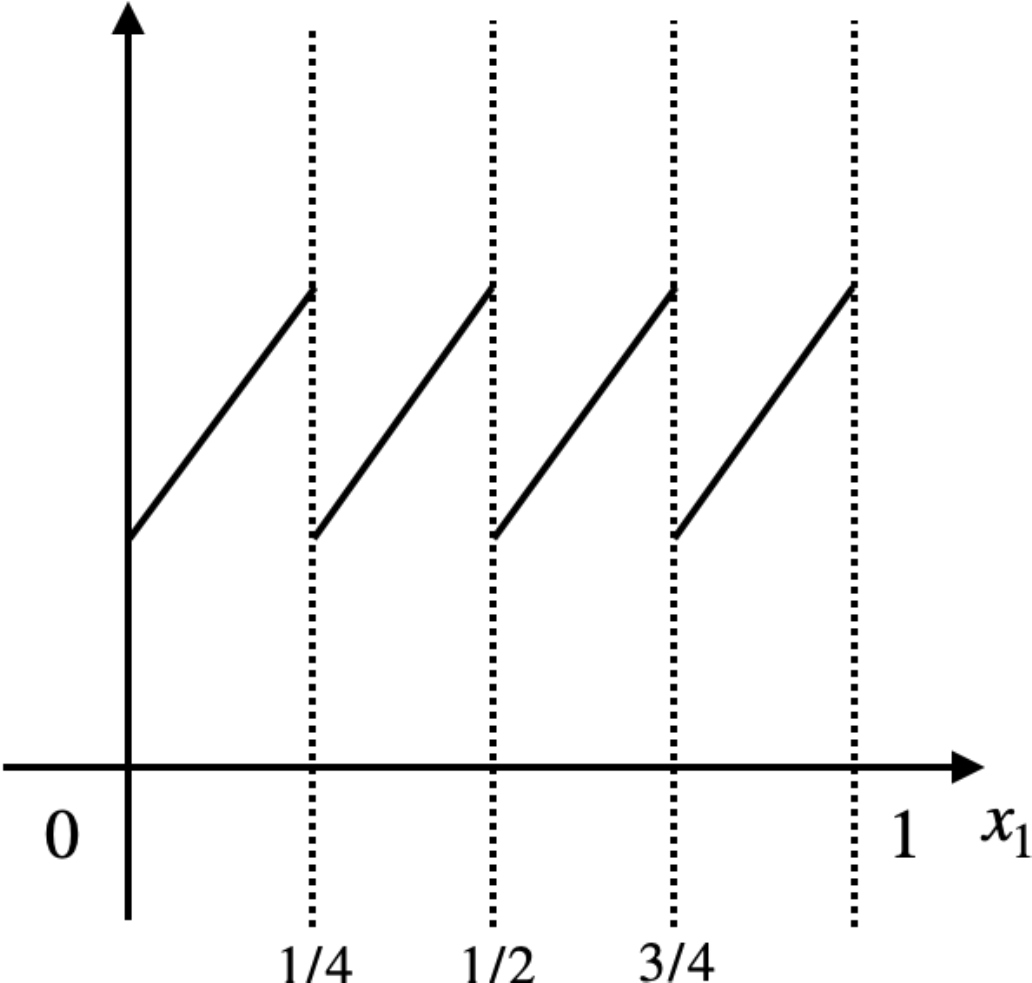}
    \caption{Example~1}
    \label{fig:periodic1}
          \end{minipage} &
    \begin{minipage}[t]{0.3\hsize}
    \centering
    \includegraphics[keepaspectratio,scale=0.35]{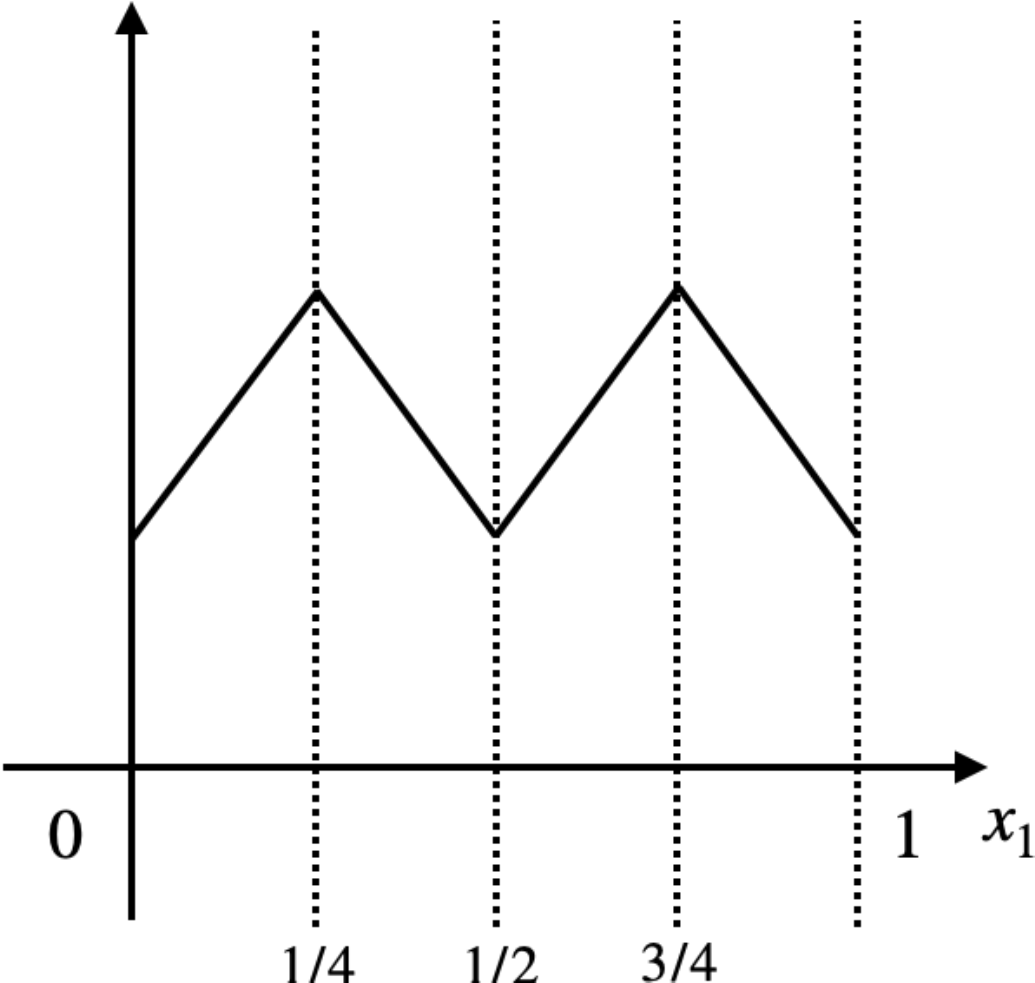}
    \caption{Example~2}
    \label{fig:periodic2}
    \end{minipage} &
      \begin{minipage}[t]{0.3\hsize}
    \centering
    \includegraphics[keepaspectratio,scale=0.35]{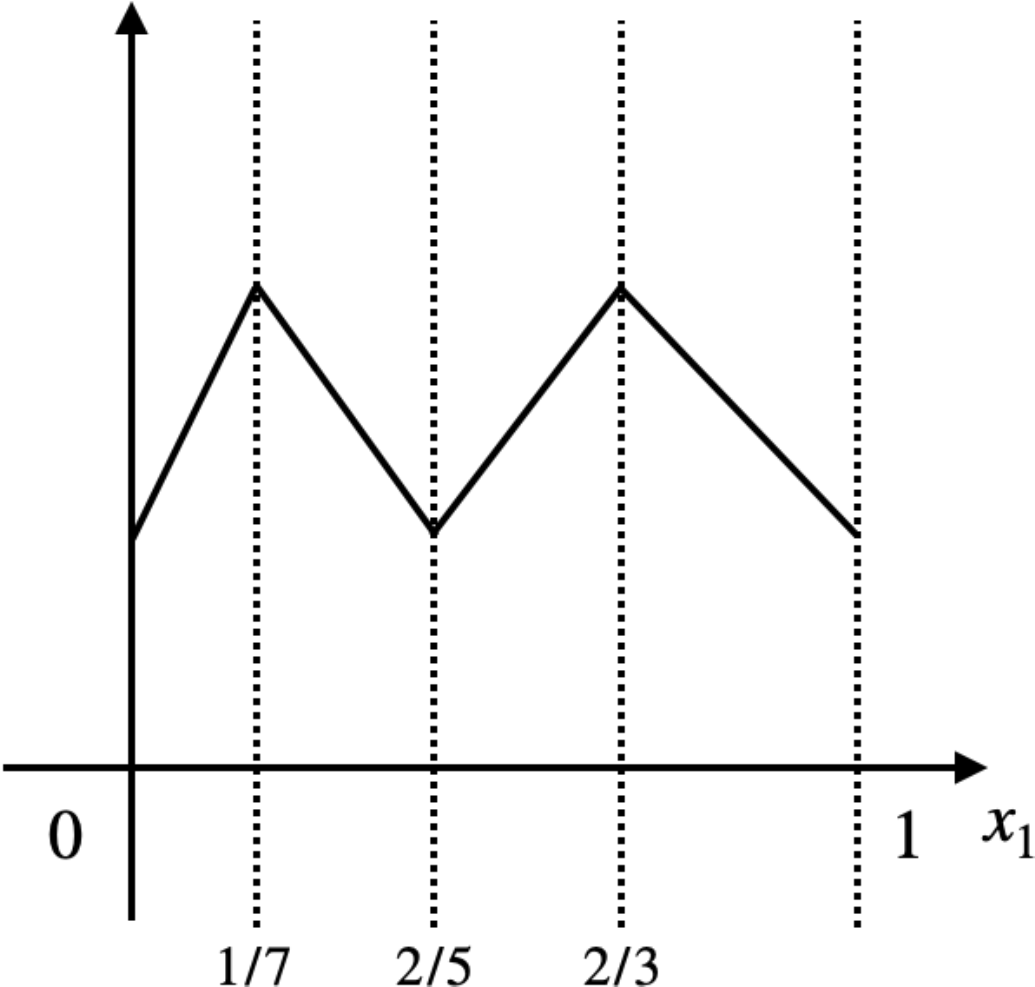}
    \caption{Example~3}
    \label{fig:bara}
      \end{minipage}
\\ \\
      \multicolumn{3}{c}{
         \begin{tabular}{cc}
      \begin{minipage}[t]{0.3\hsize}
    \centering
    \includegraphics[keepaspectratio,scale=0.35]{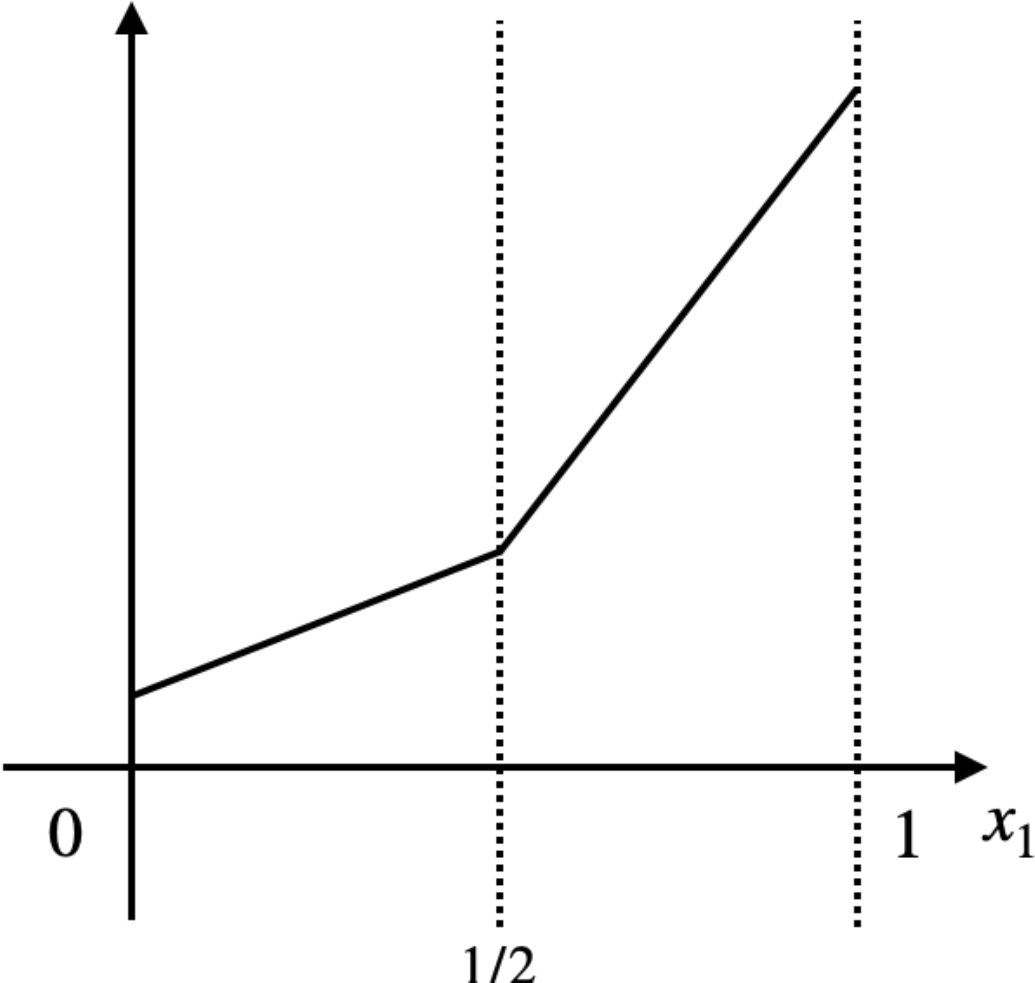}
    \caption{Example~4}
    \label{fig:inequivalent}
          \end{minipage} 
    \begin{minipage}[t]{0.3\hsize}
    \centering
    \includegraphics[keepaspectratio,scale=0.35]{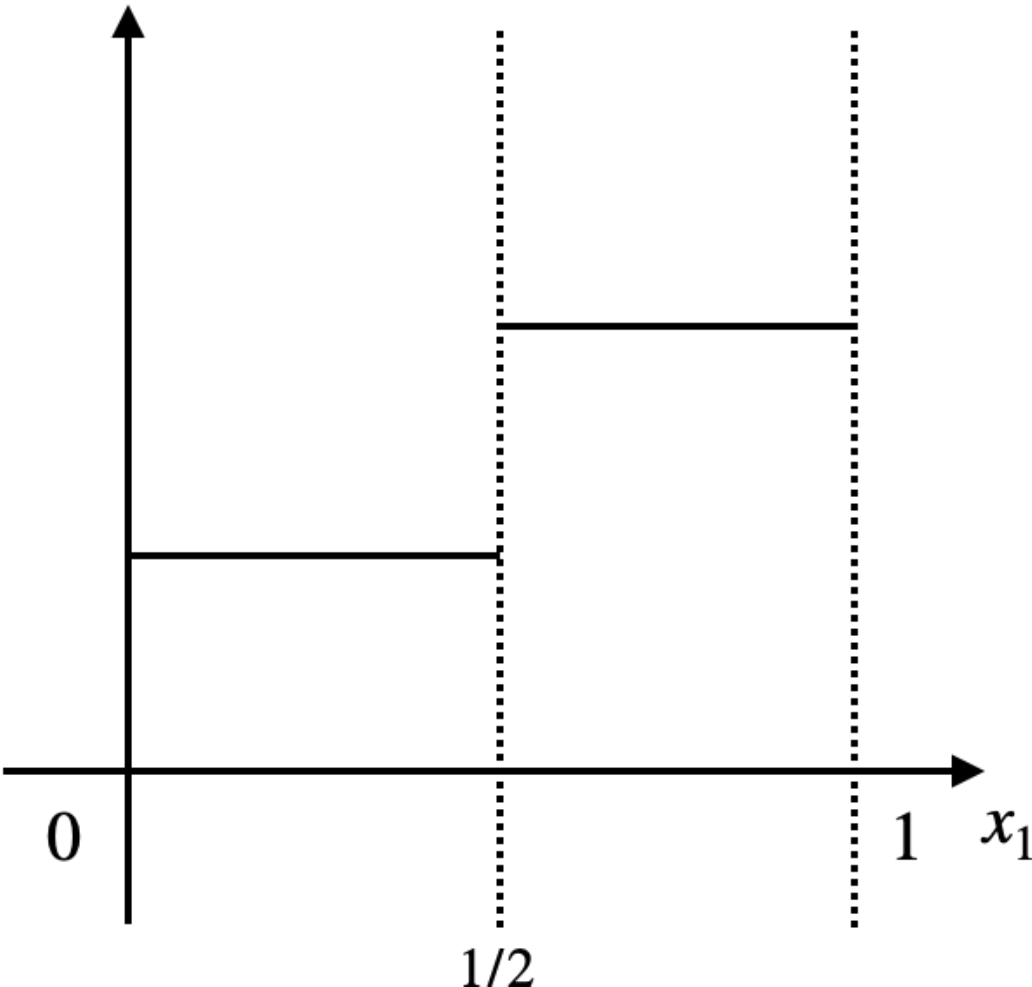}
    \caption{Example~5}
    \label{fig:step-function}
    \end{minipage}
    \end{tabular}
      }
    \end{tabular}
\end{figure*}

\begin{ex}
Let $f\colon [0,1] \to \RR$ be the piecewise linear function defined 
by $\{ (f_i, D_i) , i=1,2,3,4 \}$, where $D_1 = [0,1/4]$, $D_2=[1/4, 1/2]$, $D_3 = [1/2, 3/4]$, $D_4 = [3/4, 1]$ and 
\begin{align*}
	\begin{cases}
		f_1(x) = \alpha x + \beta & \mathrm{for } \ x \in D_1, \\
		f_2(x) = \alpha(x-1/4) + \beta & \mathrm{for } \ x \in D_2, \\ 
		f_3(x) = \alpha(x-1/2) + \beta & \mathrm{for } \ x \in D_3, \\ 
		f_4(x) = \alpha(x-3/4) + \beta & \mathrm{for } \ x \in D_4. \\ 
	\end{cases}
\end{align*} 
The function is drawn on Figure~\ref{fig:periodic1}. 
By the Euclidean transformation $\phi\colon [0,1/4] \to [1/4, 1/2]$, 
$x\mapsto x+1/4$, $(f_1, D_1) \sim (f_2, D_2)$ holds. 
Similarly, $(f_i, D_i) \sim (f_1, D_1)$ holds for any $i$. 
Hence, $c^\sim(f) = 1$.  

\end{ex}

\begin{ex}
Let $f\colon [0,1] \to \RR$ be the piecewise linear function defined 
by $\{ (f_i, D_i) , i=1,2,3,4 \}$, where $D_1 = [0,1/4]$, $D_2=[1/4, 1/2]$, $D_3 = [1/2, 3/4]$, $D_4 = [3/4, 1]$ and 
\begin{align*}
	\begin{cases}
		f_1(x) = \alpha x + \beta & \mathrm{for } \ x \in D_1, \\
		f_2(x) = -\alpha(x-1/4) + \alpha/4 +\beta & \mathrm{for } \ x \in D_2, \\ 
		f_3(x) = \alpha(x-1/2) + \beta & \mathrm{for } \ x \in D_3, \\ 
		f_4(x) = -\alpha(x-3/4) + \alpha/4 + \beta & \mathrm{for } \ x \in D_4, \\ 
	\end{cases}
\end{align*} 
as Figure~\ref{fig:periodic2}. 
By the Euclidean transformation $\phi\colon [0,1/4] \to [1/4, 1/2]$, 
$x\mapsto -x+1/2$, $(f_1, D_1) \sim (f_2, D_2)$ holds. 
Similarly, $(f_i, D_i) \sim (f_1, D_1)$ holds for any $i$. 
Hence, $c^\sim(f) = 1$.  

\end{ex}

\begin{ex}
Let $f\colon [0,1] \to \RR$ be the piecewise linear function defined 
by $\{ (f_i, D_i) , i=1,2,3,4 \}$, where $D_1 = [0,1/7]$, $D_2=[1/7, 2/5]$, $D_3 = [2/5, 2/3]$, $D_4 = [2/3, 1]$ as  Figure~\ref{fig:bara}. 
Then, we have $|D_1| = 1/7$, $|D_2| = 9/35$, $|D_3|=4/15$, and $|D_4|=1/3$. 
By the above argument in the beginning of this section, 
there is no Euclidean transformation $\phi$ such that $\phi(D_j)=D_i$ 
for any $i\neq j$. 
Hence, $c^\sim(f) = 4$.   

\end{ex}

\begin{ex}
Let $f\colon [0,1] \to \RR$ be the piecewise linear function defined 
by $\{ (f_i, D_i) , i=1,2 \}$, where $D_1 = [0,1/2]$, $D_2=[1/2, 1]$,  
\begin{align*}
	\begin{cases}
		f_1(x) = \alpha_1 x + \beta_1 & \mathrm{for } \ x \in D_1, \\
		f_2(x) = \alpha_2 x +\beta_2 & \mathrm{for } \ x \in D_2, 
	\end{cases}
\end{align*} 
such that $|\alpha_1| \neq |\alpha_2|$ as  Figure~\ref{fig:inequivalent}. 
Then, there is no Euclidean transformation $\phi$ such that $\phi(D_1)=D_2$. 
Hence, $c^\sim(f) = 2$.   

\end{ex}

\begin{ex}
Let $f\colon [0,1] \to \RR$ be the piecewise linear function defined 
by $\{ (f_i, D_i) , i=1,2 \}$, where $D_1 = [0,1/2]$, $D_2=[1/2, 1]$,  
\begin{align*}
	\begin{cases}
		f_1(x) = \beta_1 & \mathrm{for } \ x \in D_1, \\
		f_2(x) = \beta_2 & \mathrm{for } \ x \in D_2, 
	\end{cases}
\end{align*} 
such that $\beta_1 \neq \beta_2$ as  Figure~\ref{fig:step-function}. 
Then, there is no Euclidean transformation $\phi$ such that $\phi(D_1)=D_2$. 
Thus, $c^\sim(f) = 2$.   

\end{ex}

\subsection{Fully connected shallow models} 
\label{sec:complexity-shallow-full}

In this subsection, we show the existence of a fully connected 
model as \eqref{eq:shallow-fully-connected-model} 
for which the proposed complexity is equal to 
$\sum_{i=0}^{n_0} {n_1 \choose i}$. 
Let $F$ be a ReLU shallow neural network model as 
\eqref{eq:shallow-fully-connected-model} and 
$\ca{F}(F) = \{ (F_i, D_i) \mid i = 1, \dots, N \}$ be 
the set of linear functions of $F$.  
As remarked above, by the condition of Definition~\ref{def:complexity} (1), if the volumes of two linear regions $D_i$ and $D_j$ 
are different, the corresponding linear functions $F_i$ and $F_j$ cannot be equivalent. 
Therefore, if all the linear regions $D_i$ have different volumes, 
all the equivalence class of $F$ are singletons, and its complexity $c^{\sim}(F)$  is equal to the number 
$N$ of linear regions.  
By perturbing the weight matrix $W$ or the bias vector 
$\bm{c}$, we can make $F$ satisfy this condition. 
Hence, the measure of the complexity $c^{\sim}(\ca{H}_K^{\mathrm{full}}(n_0, n_1, n_2))$ 
of fully connected shallow ReLU neural networks remains equal to
$\sum_{i=0}^{n_0} {n_1 \choose i}$.

\begin{thm}\label{comp full}
The measure of complexity of $\ca{H}_K^{\mathrm{full}}(n_0, n_1, n_2)$ 
is equal to 
$c^{\#}(\ca{H}_K^{\mathrm{full}}(n_0, n_1, n_2))$. In particular, 
the following holds: 
 \[
  c^{\sim}(\ca{H}_K^{\mathrm{full}}(n_0, n_1, n_2)) \geq 
  \frac{2^{n_0 H(n_1/n_0)}}{\sqrt{8n_0(1-n_0/n_1)}}. 
 \]
\end{thm}

\subsection{Permutation invariant models}\label{sec:complexity-shallow-invariant}

Next, we consider this complexity for the permutation-invariant 
model. In this case, 
the permutation action of permutation group 
$S_n$ induces equivalence on linear regions. This effect causes 
the gap between our complexity and the number of linear regions. 
In particular, for a permutation-invariant model $F$, the 
complexity $c^{\sim}(F)$ is equal to the number of orbits of 
linear regions via permutation action. 
To calculate the number of orbits of linear regions, 
we use arguments from \textit{group action stable} hyperplanes arrangement theory investigated in 
 \citep{kamiya2012arrangements}. See Appendix~\ref{subsec:Measure-of-complexity-example} for an illustration on a simple example.

Let $K$ be a connected compact $n$-dimensional subset of $\RR^n$ 
which is stable by permutation action. 
As in Section~\ref{sec:num-of-region-shallow-invariant}, 
the linear regions of the restriction to $K$ of permutation-invariant model $F$ defined in 
\eqref{eq:shallow-permutation-invariant}, are the chambers of the 
hyperplanes arrangement 
$\ca{B}_{m,n} = \{ H_{ij} \mid i =1, \dots, m, j = 1, \dots, n \}$
defined in \eqref{eq:invariant-equations}. 
Then, $\ca{B}_{m,n}$ is stable by the permutation action, i.e., 
for any $\sigma \in S_n$, $\sigma(H_{ij}) = H_{i\sigma^{-1}(j)}$ 
holds, where $\sigma(H_{ij}) = \{ \sigma\cdot\bm{x} \mid \bm{x}\in H_{ij} \}$. 
Then, the set of chambers $\Ch(\ca{B})$ is also stable by 
permutation action. We remark that the measure of complexity 
$c^{\sim}(\ca{H}_K^\text{inv}(n,mn,m'))$ is equal to the maximum number of orbits of $\Ch(\ca{B})$, because by perturbing weight matrix or bias, we may assume that any two chambers in different orbits have different volumes.   
We set $\ca{A}_n$ to be the arrangement 
$\{ W_{ij} \mid 1\leq i < j \leq n\}$ called  the \textit{Coxeter arrangement} of $S_n$, 
where $W_{ij}$ is the hyperplane defined by the equation $x_i - x_j = 0$. 
We may assume that $\ca{A}_n \cap \ca{B}_{m,n} = \emptyset$ 
by perturbing weight matrix or bias vector if required. 
Let $\ca{C}_{m,n}= \ca{A}_{n}\cup \ca{B}_{m,n}$. 
Then, by \citep[Th.~2.6]{kamiya2012arrangements}, 
the following holds: 
\begin{thm}
\label{thm:arrangement-group-action}
 The number of orbits of $\Ch(\ca{B}_{m,n})$ with respect to permutation action 
 is equal to $|\Ch(\ca{C}_{m,n})|/n!$. 
\end{thm}

This theorem allows us to reduce the calculation of the number of 
\textit{orbits of chambers} of $\Ch(\ca{B}_{m,n})$ to the calculation of the 
number $|\Ch(\ca{C}_{m,n})|$ of chambers of $\Ch(\ca{C}_{m,n})$.  
This can be calculated inductively using 
the Deletion-Restriction theorem (Theorem~\ref{thm:deletion-restriction} in 
Appendix~\ref{sec:proof-num-of-chambers-invariant}). 
Then, we obtain the following estimate of the complexity of 
permutation-invariant shallow model: 
\begin{thm}\label{thm:complexity-invariant-shallow}
The measure of complexity of $\ca{H}_K^\mathrm{inv}(m,mn,m')$ satisfies 
$c^{\sim}(\ca{H}_K^\mathrm{inv}(m,mn,m')) \leq (n+\alpha)!/\alpha!n!$. 
Here, $\alpha = 2^{mH(1/m)}$ and 
$\gamma!$ for a positive real number $\gamma$ is the 
generalized factorial defined by 
$\gamma! = \prod_{0\leq k < \gamma} (\gamma - k)$. 
\end{thm}


\begin{proof}
We set $c^k_{n}$ as the numbers of the chambers of the hyperplane 
arrangement $\ca{C}^k_{m,n} = \ca{A}_k\cup \ca{B}_{m,n}$ for 
\[
 \ca{A}_k = \{ W_{ij} \mid 1\leq i<j \leq k \}. 
\]
This $\ca{A}_k$ can be regarded as the Coxeter arrangement for $S_k$. 
Using this notation, it is straightforward to demonstrate that 
$c^k_n$ satisfies the following recurrence relation: 
\[
 c^k_n = c^{k-1}_n + k c^{k-1}_{n-1}. 
\]
Using this relation, we have 
\[
 |\Ch(\ca{C}_{m,n})| = c^n_n = \sum_{l=0}^n \left( \sum_{1\leq k_1 < \cdots < k_l \leq n} k_1\cdots k_l \right) c^0_{n-l}. 
\]
If we use the upper bound of 
$c^0_{n-l} \leq \alpha^{n-l}$, 
where $\alpha = 2^{mH(1/m)}$ as in 
\eqref{eq:estimate-linear-region-fully-connected}, we have 
\begin{align*}
 |\Ch(\ca{C}_{m,n})| &\leq \sum_{l=0}^n \left( \sum_{1\leq k_1 < \cdots < k_l \leq n} k_1\cdots k_l \right) \alpha^{n-l}  \\ 
 &= \prod_{k = 1}^n (\alpha + k) = \frac{(n+\alpha)!}{\alpha!}. 
\end{align*}
Hence, by combining this and Theorem~\ref{thm:arrangement-group-action}, the number of orbits of $\Ch(\ca{B}_{m,n})$ is bounded from above as 
\begin{align*}
 \text{(the number of orbits of }\Ch(\ca{B}_{m,n})\text{)} &= 
 \frac{|\Ch(\ca{C}_{m,n})|}{n!} \\ &\leq \ \frac{(n+\alpha)!}{\alpha! n!}. \qedhere
\end{align*}
\end{proof}

\subsection{Comparison of the measures between fully connected and permutation invariant models}   

We compare these complexities between fully connected shallow model 
and permutation invariant shallow model. To equalize the number of hidden units in both models, we consider $n_0 = n$ and $n_1 = mn$. 
Let $K$ be a connected compact $n$-dimensional subset 
$K\subset \RR^n$ which is stable by permutation action. 
Then, because the maximum number of equivalent classes for fully connected shallow 
models is bounded from below by $\alpha^n/{\sqrt{8n(1-1/m)}}$ as in 
\eqref{eq:estimate-linear-region-fully-connected}, where $\alpha = 2^{mH(1/m)}$. 
This means that the measure of complexity increases exponentially when 
$n$ increases.  
Meanwhile, by Theorem~\ref{thm:complexity-invariant-shallow}, 
the maximum number of equivalent classes for permutation 
invariant shallow models is bounded from above by 
\[
\frac{(n+\alpha)!}{\alpha!n!} \leq \frac{(n+\alpha)(n+\alpha - 1) \cdots (n+\alpha - \lfloor \alpha \rfloor )}{\alpha!}. 
\] 
In the second inequality, we used the fact that 
$n+\alpha -\lfloor\alpha \rfloor -k \leq n-k + 1$. 
By this argument, the measure of complexity 
$c^{\sim}(\ca{H}_K^{\text{inv}}(m,mn,n'))$ of the set of the 
permutation invariant shallow models is bounded from above by a 
polynomial with respect to $n$ of degree 
$\lfloor \alpha \rfloor + 1 $. By comparing these measures, we have 
\begin{align*}
c^{\sim}(\ca{H}_K^{\mathrm{inv}}(n, mn, n')) &\leq 
\frac{(n+\alpha)!}{\alpha!n!} \\ &\ll 
  \frac{\alpha^n}{\sqrt{8n(1- 1/m)}} \\ 
  &\leq 
  c^{\sim}(\ca{H}_K^{\mathrm{full}}(n, mn, n')).  
\end{align*}
In particular, 
$c^{\sim}(\ca{H}_K^{\text{inv}}(m,mn,n'))$ is strictly smaller than 
$c^{\sim}(\ca{H}_K^{\mathrm{full}}(m,mn,n'))$.  
Therefore, the proposed complexity behaves better to evaluate expressive power than simply counting linear regions. 

\section{Specific deeper models}
\label{sec:deeper-models}

In this section, we provide a variant of the model which has 
been introduced by \citet{montufar2014number} 
and show that this can be used to confirm that deep models can have much higher complexity than shallow models. 

\subsection{A variant of the model of Mont\'{u}far et al} 
\label{sec:deeper-fully-connected}
We here introduce a variant of the model of 
\citet{montufar2014number}. 
The original model introduced by 
\citet{montufar2014number} 
is a deep neural network defined by some special affine maps designed to cause ``folding'' efficiently. 
From the way it is constructed, the hidden layers divide the input space into a grid of hypercubes, and the division into linear regions produced by the output layer is copied into each hypercube.
We modify this model to be able to control the lengths of sides of the hypercubes to obtain hypercuboids which have different volumes.    


The model is defined as follows: 
We consider a neural network of depth $L+1$ and width as \eqref{eq:deep-fully-connected}.  
We assume that $n \leq n_l$ for any $l$ and 
set $p_l = \lfloor n_l/n \rfloor$. 
For $j \in \{1, 2, \dots, n \}$, 
we set $\bm{w}_j^\top = (0, \dots, 0, 1 , 0, \dots, 0)$ as the 
vector $\bm{w}_j \in \RR^{n}$ whose $j$-th entry is $1$ and 
the others are $0$. 
For $l = 1, 2, \dots, L-1$, 
we define $\wt{h}^{(l)} \colon \RR^n \to \RR^n$ as follows. 
We take, for any $j= 1, 2, \dots, n$, 
positive integers $a_1^{j, (l)}, \dots, a_{p_l}^{j, (l)}$ satisfying 
$\sum_{k=1}^{p_l} a_k^{j, (l)} =1$ and set 
$b_k^{j, (l)} = (a_k^{j, (l)})^{-1}$ and 
\[
 c^{j, (l)}_k = \begin{cases}
 	-b_k^{j, (l)} (a_1^{j, (l)} + \cdots + a_k^{j, (l)}) & \text{if } k \text{ is even},\\ 
 	-b_k^{j, (l)} (a_1^{j, (l)} + \cdots + a_{k-1}^{j, (l)}) & \text{if } k \text{ is odd}.
 \end{cases}
\]
For $j = 1, 2, \dots, n $ and $k = 1, \dots, p_l $, 
we define the function 
$h^{j, (l)}_k \colon \RR^n \to \RR$ as 
\begin{align*}
 &h^{j, (l)}_k(\bm{x}) = \\ 
 &\begin{cases} 
 	\max\{ 0, b^{j, (l)}_1\bm{w}_j^\top \bm{x} \} & \text{if} \ k = 1, \\
  \max\{ 0, (b^{j, (l)}_{k-1}+b^{j, (l)}_k)\bm{w}_j^\top \bm{x} + \sum_{s=2}^k c^{j, (l)}_s \} & \text{if} \ k\geq 2.  
 \end{cases}
\end{align*}


Using these $h^{j,(l)}_k$, we define the map 
$\wt{h}^{j,(l)} \colon \RR^n \to \RR$ as 
\begin{align}
 \wt{h}^{j,(l)} (\bm{x}) &= 
 \sum_{k = 1}^p (-1)^{k-1} h^{j,(l)}_k(\bm{x}). \label{eq:h^{j,(l)}} 
\end{align}
Then, as in Figure~\ref{fig:1}, for $\bm{x}\in \RR^n$ such that 
$a_1^{j,(l)} + \dots + a_{i-1}^{j,(l)} \leq x_j < a_1^{j,(l)} + \dots + a_i^{j,(l)}$, 
$\wt{h}^{j,(l)}(\bm{x})$ satisfies 
\[
 \wt{h}^{j,(l)} (\bm{x}) = (-1)^{i+1} (b^{j,(l)}_i x_j + c^{j,(l)}_i).   
\]
\begin{figure}[t]
        \centering
        \includegraphics[keepaspectratio, scale=1.6]{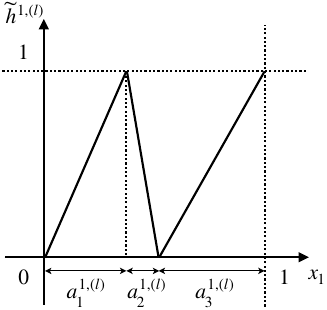}
        \caption{The graph of $\wt{h}^{1,(l)}$ for $p=3$}
        \label{fig:1}
  \end{figure}
\begin{figure}[t]
        \centering
        \includegraphics[keepaspectratio, scale=1.6]{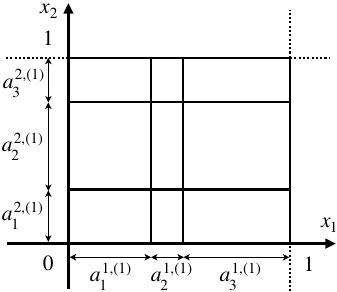}
        \caption{A decomposition of $[0,1]^2$ into rectangles ($2$-dim hypercuboids) by $\wt{h}^{(l)}$ for $n=2$, $p=3$}
        \label{fig:2}
  \end{figure}

We remark that although the input space of $\wt{h}^{j,(l)}$ is $\RR^n$, 
this map depends only on $j$-th entry of $\bm{x}$. Hence, 
we can regard this as 
from $\RR$ to $\RR$. 
Moreover, this map $\wt{h}^{j,(l)}$ divides the subinterval $[0,1]$ of 
$x_j$-axis into
 $p_l$ regions $(-\infty, 0]$, $[0, a_1^{j,(l)}]$, $[a_1^{j,(l)}, a_1^{j,(l)} + a_2^{j,(l)}]$, $\dots$, 
$[\sum_{i =1}^{p_l-1}a_i^{j,(l)}, \infty)$ and the image of each regions by $\wt{h}^{j,(l)}$ is $[0,1]$. 
This construction makes a $p_l$-fold ``folding''.

We define $\wt{h}^{(l)}\colon \RR^n \to \RR^{n}$ by  
$\wt{h}^{(l)} = (\wt{h}^{1,(l)}, \dots, \wt{h}^{n,(l)})^\top $. 
By the construction, this map $\wt{h}^{(l)}$ can be 
realized as a ReLU neural network as 
\begin{align*}
 &\RR^n \to \RR^{n_l} \to \RR^n;\\ 
& \bm{x} \mapsto 
 (h_1^{1,(l)}(\bm{x}), \dots, h^{n,(l)}_{p_l}(\bm{x}), 0, \dots, 0)^\top \\ &\mapsto 
 (\wt{h}^{1,(l)}(\bm{x}), \dots, \wt{h}^{n,(l)}(\bm{x}))^\top. 
\end{align*}
This map $\wt{h}^{(l)}$ divides $[0,1]^n \subset \RR^n$ into 
$p_l^n$ $n$-dimensional hypercuboids. 
We remark that the volume of the $(i_1, \dots, i_n)$-th 
hypercuboid is $a^{1,(l)}_{i_1}a^{2,(l)}_{i_2} \cdots a_{i_n}^{n,(l)}$ as in Figure~\ref{fig:2}. 

Then, the composition 
$\wt{h}^{(L-1)} \circ \cdots  \circ 
\wt{h}^{(1)}$ defines the deep neural network of depth $L$, 
width $n_0, n_1, \dots, n_{L}$, and output $\RR^n$. 
This map sends $[0,1]^n \in \RR^{n}$ to 
$[0,1]^n \subset \RR^n$ and divides $[0,1]^n$ into the 
$(p_1 p_2 \cdots p_{L-1})^n$ $n$-dimensional hypercubes with 
linear regions as in Figure~\ref{fig:3}.  
Then, the volume of $(\ol{i}_1, \ol{i}_2, \dots, \ol{i}_n)$-th hypercube is 
\begin{align*}
 \left(a_{i_{1,L-1}}^{1, (L-1)} \cdots a_{i_{11}}^{1, (1)}\right)
 &\cdot
 \left(a_{i_{2,L-1}}^{2, (L-1)} \cdots a_{i_{21}}^{2, (1)}\right)\cdot \\  &\cdots \cdot 
 \left(a_{i_{n,L-1}}^{n, (L-1)} \cdots a_{i_{n1}}^{n, (1)}\right), 
\end{align*}
where 
$\ol{i}_k = (i_{k1}, \dots, i_{k,L-1}) \in \prod_{l=1}^{L-1}
\{1, \dots, p_l \}$. 
\begin{figure*}[t]
        \centering
        \includegraphics[keepaspectratio, scale=1.4]{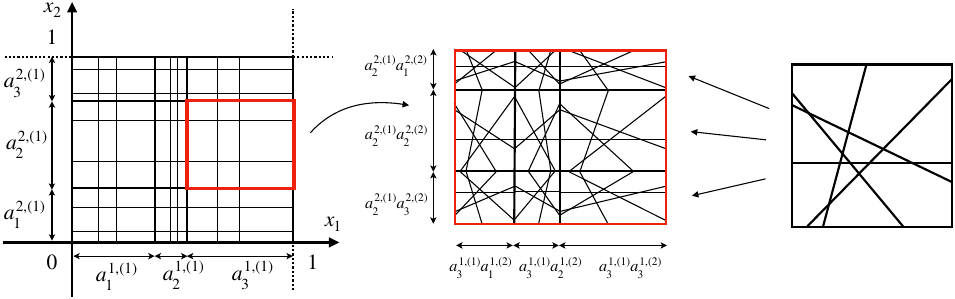}
        \caption{A grid decomposition of $[0,1]^2$ into rectangles by $\wt{h}^{(2)}\circ \wt{h}^{(1)}$ for $n=2, p_1=p_2 =3$ and an image of copies of hyperplane arrangement in $[0,1]^2$ 
        by $F$ into the rectangles}
        \label{fig:3}
\end{figure*} 

In particular, by perturbing weights if necessary, 
we may assume that any hypercuboids have different volumes. 

Next, we choose a map $F\colon \RR^n \to \RR^{n_{L}}$ 
which gives a hyperplane arrangement whose chambers have different volumes introduced in Section~\ref{sec:complexity-shallow-full} 
and by scaling, we assume that all the intersections of 
hyperplanes are in the interior of the hypercube 
$[0,1]^n \subset \RR^n$. 
Finally, we take the composition 
$\wt{h}^{(L-1)} \circ \cdots  \circ 
\wt{h}^{(1)}$ with $F$. 
Then, the hyperplanes arrangement in $[0,1]^n$ defined by $F$ 
is copied into each hypercuboids as in Figure~\ref{fig:3}. 
If we need, by perturbing weights again, 
we may assume that any linear region has different volume. 
This implies that the measure of complexity $c^{\sim}(F\circ\wt{h}^{(L-1)} \circ \cdots  \circ \wt{h}^{(1)})$ coincides with the maximum of the 
number of linear regions. 
In particular, this is equal to 
\[
 \prod_{i=1}^{L-1} \left(\left\lfloor \frac{n_i}{n}\right\rfloor\right)^n \left( \sum_{k = 0}^{n} 
 {	n_{L} \choose  k} \right). 
\]
This shows the following: 
\begin{thm}
The measure of complexity $c^{\sim}(\ca{H}_{[0,1]^n}^{\mathrm{full}}(n_0, n_1, \dots, n_L, n_{L+1}))$ 
for the model of above defined ReLU deep neural networks is bounded from below by 
$ \prod_{i=1}^{L-1} \left(\left\lfloor \frac{n_i}{n}\right\rfloor\right)^n \left( \sum_{k = 0}^{n} 
 {	n_{L} \choose  k} \right)$.
\end{thm}


As a consequence of the arguments of Section~\ref{sec:complexity-shallow-full} and this section, both of 
the complexities for fully connected models which 
appear there are same as maximum numbers of linear 
regions. Hence, by similar argument to \citet{montufar2014number}, 
the complexity of deeper models is exponentially larger than 
the shallow models. 
This also shows the benefit of depth for neural network.


%


\subsection{A benefit of depth for deep set models} 


We here consider a permutation-invariant deep model, called 
{\em deep set model} introduced by \citet{zaheer2017deep}. 
This model is made by stacking some permutation equivariant maps and 
one invariant map. Thus, the obtained map is permutation-invariant.   
This model has some common features with the model of 
\citet{montufar2014number}. 
Indeed, the original model of Mont\'{u}far et al, except for the map from the last hidden layer to the output layer, is equivalent to the 
deep set model. 
We shall modify the model variant introduced in Section~\ref{sec:deeper-fully-connected} to be a deep set model 
and show that deep set models also have a similar benefit of depth.  

As mentioned above, the deep set model is defined by stacking permutation 
equivariant affine maps and one invariant map. 
More specifically, the ReLU deep neural network 
$f_{L+1} \circ \ReLU \circ f_{L} \circ  \cdots \circ \ReLU \circ f_1$ for affine maps $f_i \colon (\RR^{n})^{m_{i-1}} \to 
(\RR^{n})^{m_i}$ is called {\em a deep set model} if $f_1, \dots, f_{L}$ are permutation equivariant and 
$f_{L+1}\colon (\RR^n)^{m_L} \to \RR^{m_{L+1}}$ is permutation-invariant. 
For $\ul{m} =(m_1, m_2, \dots, m_L, m_{L+1})$, let $\ca{H}_{[0,1]^n}^\text{inv}(n, \ul{m})$ 
be the set of the restrictions to $[0,1]^n$ of the 
deep set models. 

If we assume that the variant of model of \citet{montufar2014number} which we introduced  
in Section~\ref{sec:deeper-fully-connected} satisfies that 
$ h^{1,(l)}_k = \cdots =  h^{n,(l)}_k$ for any $k$ and any $l$, and 
that $F\colon (\RR^n)^{m_L} \to \RR^{m_{L+1}}$ is 
permutation-invariant, then the obtained neural network 
$F\circ \wt{h}^{(L-1)}\circ \cdots \circ\wt{h}^{(1)}$ is in 
$\ca{H}_{[0,1]^n}^\text{inv}(n,\ul{m})$. 
In this case, $a_{k}^{1, (l)}= \cdots = a_{k}^{n,(l)}$ holds for any $k$ and $l$. We set $a_{k}^{(l)}$ to be this number. 
The obtained neural network providing the 
$(\prod_{i=1}^{L-1} p_l)^n$ $n$-dimensional hypercuboids. However, 
the volume of $(\ol{i}_1, \ol{i}_2, \dots, \ol{i}_n)$-th hypercuboid 
 is 
{\small 
\begin{align*}
 \left(a_{i_{1,L-1}}^{(L-1)} \cdots a_{i_{11}}^{(1)}\right) \cdot
 \left(a_{i_{2,L-1}}^{(L-1)} \cdots a_{i_{21}}^{(1)}\right) \cdots 
 \left(a_{i_{n,L-1}}^{(L-1)} \cdots \cdot a_{i_{n1}}^{(1)}\right),
\end{align*}
}
where 
$\ol{i}_k = (i_{k,1}, \dots, i_{k,L-1}) \in \prod_{l=1}^{L-1}
\{1, \dots, p_l \}$. 
We regard the index set 
$\prod_{l=1}^{L-1}\{1, \dots, p_l \}$ as an ordered set 
by the lexicographic order $\leq$. 
Then, by perturbing the weights or biases if we need, 
we may assume that 
any hypercuboid in the set of hypercuboids whose index 
$(\ol{i}_1, \ol{i}_2, \dots, \ol{i}_n)$ satisfies 
$\ol{i}_1 < \ol{i}_2 < \dots < \ol{i}_n$ have different volumes, 
and the number of such hypercuboids is 
${p_1\cdots p_{L-1} \choose n}$. 
We choose the affine map from $\RR^n$ to output layer 
$(\RR^n)^{m_{L}}$ to be the one which achieves the measure of 
complexity of $\ca{H}_{[0,1]^n}^\text{inv}(n, nm_{L}, m_{L+1})$ as in 
Section~\ref{sec:complexity-shallow-invariant}. 
Hence, the measure of complexity of $\ca{H}^\text{inv}(n,\ul{m})$ is 
bounded from below by 
$C\cdot(m_1\cdots m_{L})^n n^n/n!$ for a positive constant $C$. In particular, the following holds: 
\begin{thm} The following holds: 
 $c^{\sim}(\ca{H}_{[0,1]^n}^{\mathrm{inv}}(n, \ul{m})) = 
 \Omega((m_1\cdots m_{L})^n (n^n/n!))=\Omega(
 (m_1\cdots m_L e)^n /\sqrt{n})$.
\end{thm}
We compare this with the shallow invariant model having same number of hidden units. Then, the width of the hidden layer of the 
shallow model is equal to $n\sum_{i = 1}^L m_i$. By the argument 
in Section~\ref{sec:num-of-region-shallow-invariant}, the measure of 
complexity is $\Theta( (\sum_{i = 1}^L m_i)^n\cdot e^n /\sqrt{n})$. 
This yields that for the deep set model, deeper models can obtain 
exponentially more complexity than shallow models in our measure.

\section{Conclusion}

In this paper, we defined a new measure of complexity of ReLU neural networks, which is closer to expressive power than the number of linear regions. 
 Specifically, we considered fully connected and 
 Permutation-invariant models as examples, which are indistinguishable from the conventional measure of linear regions but have different expressive power. 
 The new complexity is introduced as the number of equivalence classes 
  that identify linear regions and linear functions on them with those transferred by a Euclidean transformation. Considering that, we have shown that the values of the measure for the two networks above are actually different. 
 In this sense, the proposed measure of complexity can be considered to represent the expressive power of the function more closely. We also proved that the value of the proposed measure increases exponentially for deeper networks by refining the model of \citet{montufar2014number} for both the fully connected model and the deep set model.


\acknowledgments{The authors would like to thank the anonymous reviewers for their suggestions and helpful comments. 
This work was supported in part by the Grant for Basic Science Research Projects from The Sumitomo Foundation (No.200484) and the JSPS Grant-in-Aid for Scientific Research C (20K03743).}


\newpage
\onecolumn
\appendix
\section{Illustrations and examples}
\label{sec:appendix-illustrations-examples}


\subsection{ Fully connected shallow model}
\label{subsec:Fully-connected-shallow-example}

In this section, we illustrate linear region calculations for simple examples in the plane.

In the two-dimensional plane, "general position" means that two lines always intersect and three lines never are concurrent. Let us see an example in the case $n_0=2$, $n_1=4$, i.e., four lines in the plan.
\begin{align}
\begin{cases}
	x - y + 1 = 0 & \ : \ H_{1}\\ 
 	x - y - 1 = 0 & \ : \ H_{2}\\ 
 	x + y - 2 = 0 & \ : \ H_{3}\\
   x = \frac{1}{2}& \ : \ H_{4}.
\end{cases}	
\end{align}
This arrangement is not in general position because $H_{1}$ and $H_{2}$ are parallel or $H_{1}$, $H_{3}$ and $H_{4}$ are concurrent. Its number of chambers is $9$ (Figure~\ref{fig:dense not general position})

Let us modify $H_2$ to make the arrangement being general position. Now we have:
\begin{align}
\begin{cases}
	x - y + 1 = 0 & \ : \ H_{1}\\ 
 	 y  = 1 & \ : \ H_{2}\\ 
 	x + y - 2 = 0 & \ : \ H_{3}\\
   x = \frac{1}{2} & \ : \ H_{4}.
\end{cases}	
\end{align}
Now the number of chambers is $11 = \sum_{i = 0}^{n_0} { n_1 \choose i }$. It is maximal for a 4-line arrangement in the real plane
 (Figure~\ref{fig:dense general position}). 
\begin{figure}[h]
    \begin{tabular}{cc}
      \begin{minipage}[t]{0.45\hsize}
        \centering
        \includegraphics[keepaspectratio, scale=0.5]{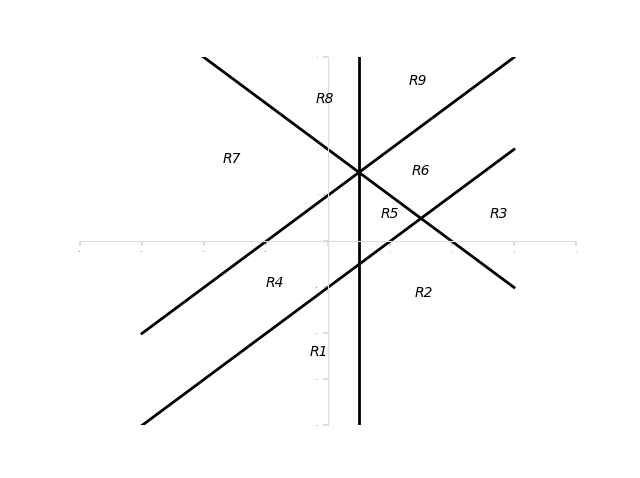}
        \caption{The line arrangement not in general position. The number of chambers is 9.}
        \label{fig:dense not general position}
      \end{minipage} &
      \begin{minipage}[t]{0.45\hsize}
        \centering
        \includegraphics[keepaspectratio, scale=0.5]{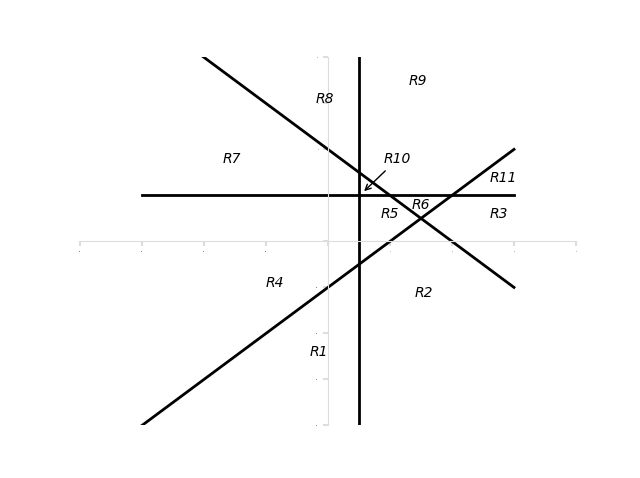}
        \caption{The line arrangement in general position. The number of chambers is 11 and is maximal}
        \label{fig:dense general position}
      \end{minipage}

    \end{tabular}
  \end{figure}

\subsection{Permutation invariant shallow model}
\label{subsec:permutation-invariant-shallow-example}

Let us consider an example of a permutation-invariant shallow model with $m=n=2$, i.e., this model also implements a function
from $\RR^2$ to $\RR^4$. We have the two pairs of lines (Figure~\ref{fig:perm_equiv_linear_regions}): 
\begin{align*}
\begin{cases}
	2x + \frac{1}{2}y - 3 = 0 & \ : \ H_{11}\\ 
 	\frac{1}{2}x + 2y - 3 = 0 & \ : \ H_{12}\\ 
 	-x + 6y  = 0 & \ : \ H_{21}\\
    6x - y = 0 & \ : \ H_{22}.
\end{cases}	
\end{align*}
\vspace{-16pt}
We also count 11 chambers.

\begin{figure}[t]
    \centering
    \includegraphics[keepaspectratio,scale=0.5]{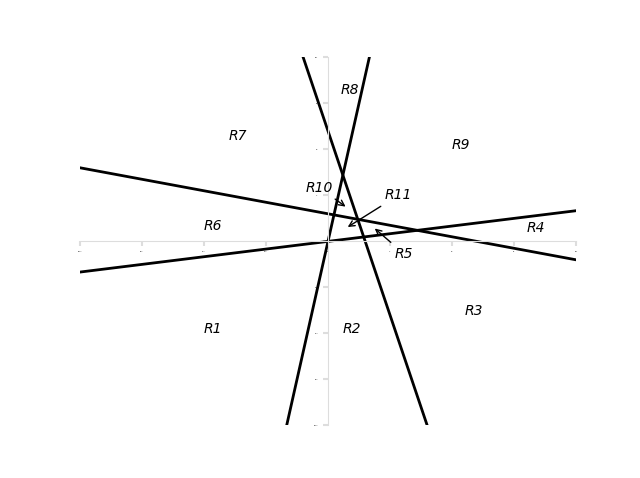}
    \caption{The 4 lines arrangement in the plane of a permutation invariant model. We count 11 linear regions. }
    \label{fig:perm_equiv_linear_regions}
\end{figure}

\subsection{Measure of complexity as the number of equivalent classes }
\label{subsec:Measure-of-complexity-example}

Let us consider again the last invariant model example: 
\begin{align}
\begin{cases}
	2x + \frac{1}{2}y - 3 = 0 & \ : \ H_{11}\\ 
 	\frac{1}{2}x + 2y - 3 = 0 & \ : \ H_{12}\\ 
 	-x + 6y  = 0 & \ : \ H_{21}\\
    6x - y = 0 & \ : \ H_{22}.
\end{cases}	
\end{align}
In this case, $S_2$ has a single element which is the permutation $\sigma = (1 \ 2)$. Here, the action of $\sigma$ on $\RR^2$ is exactly the action of the reflection symmetry through the line $x=y$. Then, the corresponding Euclidean transformation $\phi$ 
is $\phi=\begin{pmatrix}
0 & 1\\
1 & 0
\end{pmatrix}$
 and the underlying group is $\hat{\Phi}=\{ I, \phi \}$

\begin{figure}[t]
    \centering
    \includegraphics[keepaspectratio,scale=0.5]{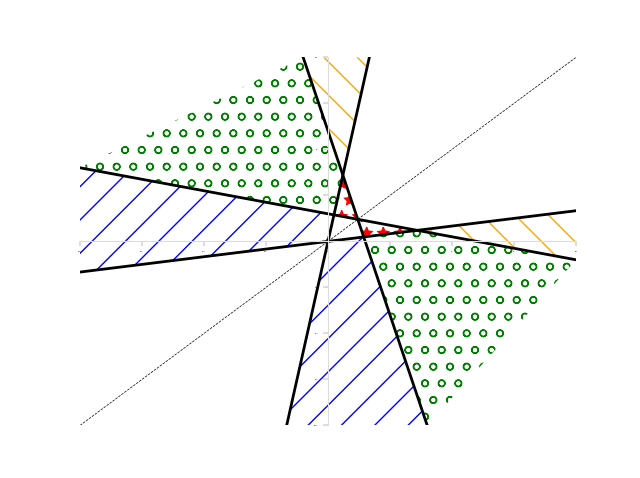}
    \caption{The dashed line is the line of equation $x=y$. We identify the equivalent regions with respect to the symmetry through the line $x=y$. The number of orbits is 7. }
    \label{fig:equivalence_classes_S2}
\end{figure}

In Figure~\ref{fig:equivalence_classes_S2}, we identify regions belonging to the same equivalent classes. In this case, a region is identified by its symmetry through the line $x=y$. Therefore, we count 7 equivalent classes of linear regions: \{R1\}, \{R2,R6\}, \{R3,R7\}, \{R4,R8\}, \{R5,R10\}, \{R9\}, \{R11\}.

\section{Proof of Proposition~\ref{prop:num-of-chambers-invariant}}
\label{sec:proof-num-of-chambers-invariant}

In this section, we prove Proposition~\ref{prop:num-of-chambers-invariant}. To show this, we use the Deletion-Restriction theorem
\cite[Theorem~2.56 and Theorem~2.68]{orlik2013arrangements}. 

\begin{thm}[Brylawsky, Zaslavsky]\label{thm:deletion-restriction}
For a hyperplane arrangement $\ca{A}$ in $\RR^n$ and a fixed hyperplane $X \in \ca{A}$, 
let $(\ca{A}, \ca{A}', \ca{A}'')$ be the triple defined as 
$\ca{A}' = \ca{A}{\setminus}\{X \}$ and 
\[
 \ca{A}'' = \{ H \cap X \mid H \in \ca{A}{\setminus}\{X \}, H\cap X \neq \emptyset \}. 
\]
Then, the following holds: 
\[
 |\Ch(\ca{A})| = |\Ch(\ca{A}')|+|\Ch(\ca{A}'')|. 
\]
\end{thm}

By apply Theorem~\ref{thm:deletion-restriction} to our hyperplane arrangement, we obtain a recurrence relation and calculate the 
number of linear regions for permutation invariant models. 

\begin{proof}[Proof of Proposition~\ref{prop:num-of-chambers-invariant}] 

Let 
$\ca{B}_{m,n} = \{ H_{ij} \subset \RR^n \mid i =1, \dots, m, \ j = 1, \dots, n \}$ 
be the hyperplane arrangement defined by 
\eqref{eq:invariant-equations}. 
We recall that hyperplanes of this arrangement $\ca{B}_{m,n}$ 
satisfy the following equations: 
\begin{align}
&H_{i_1,j}\cap H_{i_2,j}\cap H_{i_3,j} = \emptyset, \label{eq:condition-intersection-1-appendix}\\ 
&H_{i_1,j_1}\cap H_{i_1,j_2} \cap H_{i_2,j_1} = 
 H_{i_1,j_1}\cap H_{i_1,j_2} \cap H_{i_2,j_2} = 
 H_{i_1,j_1}\cap H_{i_2,j_1} \cap H_{i_2,j_2} 
 \label{eq:condition-intersection-2-appendix}
\end{align}
for $i_1, i_2, i_3 = 1, \dots, m$ and $j, j_1, j_2 = 1, \dots, n$.

We apply Theorem~\ref{thm:deletion-restriction} to $\ca{B}_{m,n}$ and 
$H_{m,n} \in \ca{B}_{m,n}$. Then, we have 
\begin{align*}
 \ca{B}_{m,n}' &= 
\{ H_{11}, \dots, H_{1n}, \dots, H_{m1}, \dots, H_{m,n-1} \}, \\
 \ca{B}_{m,n}'' &= 
 \{ H_{11}\cap H_{m,n}, \dots, H_{1n}\cap H_{m,n}, \dots, 
 H_{m1}\cap H_{m,n}, \dots, H_{m,n-1}\cap H_{m,n} \}  
\end{align*}
and $|\ca{B}_{m,n}| = |\ca{B}_{m,n}'| +|\ca{B}_{m,n}''|$
Here, because $H_{m,n}$ is a hyperplane bijective to $\RR^{n-1}$, 
 $H_{ij}\cap H_{m,n}$ can be regarded as a hyperplane in 
 $H_{m,n} = \RR^{n-1}$. 

Next, we consider deletion and restriction for  
$\ca{B}_{m.n}''$ and $H_{m-1,n}\cap H_{m,n}$. Then, we have 
\begin{align*}
 (\ca{B}_{m,n}'')' &= 
 \left\{ H_{1,1}\cap H_{m,n}, \dots, H_{m-2,n}\cap H_{m,n}, 
  H_{m-1, 1}\cap H_{m,n}, \dots, H_{m-1,n-1}\cap H_{m,n}, \atop 
 H_{m,1}\cap H_{m,n}, \dots, H_{m,n-1}\cap H_{m,n} \right\}, \\ 
 (\ca{B}_{m,n}'')'' &= \left\{ 
 H_{1,1}\cap H_{m-1,n}\cap H_{m,n}, \dots, 
 H_{1,n}\cap H_{m-1,n}\cap H_{m,n}, \dots, \atop 
 H_{m,1}\cap H_{m-1,n}\cap H_{m,n}, \dots, H_{m,n-1}\cap H_{m-1,n}\cap H_{m,n} \right\}. 
\end{align*}
Then, in the above $(\ca{B}_{m,n}'')''$, by the relation 
 \eqref{eq:condition-intersection-1-appendix}, we have  
\[
 H_{i,n}\cap H_{m-1, n}\cap H_{m,n} = \emptyset
\]
for any $i = 1, \dots, m-2$. Hence, any hyperplane of the form 
$H_{i,n}\cap H_{m-1, n}\cap H_{m,n}$ vanishes from $(\ca{B}_{m,n}'')''$. 
Moreover, by the relation 
\eqref{eq:condition-intersection-2-appendix}, 
for any $j = 1, \dots, n-1$, 
\[
 H_{m,j} \cap H_{m-1, n}\cap H_{m,n} = H_{m-1, j}\cap 
 H_{m-1,n}\cap H_{m,n}
\]
holds. By this relation, we can unify the hyperplanes of forms of  
$H_{m,j} \cap H_{m-1, n}\cap H_{m,n}$ and 
$H_{m-1, j}\cap H_{m-1,n}\cap H_{m,n}$. 
By these arguments, $(\ca{B}_{m,n}'')''$ can be written by 
\[
  (\ca{B}_{m,n}'')'' = 
  \{ H_{i,j}\cap H_{m-1,n}\cap H_{m,n} \subset \RR^{n-2} \mid 
  i =1, \dots, m-1, \ j = 1, \dots, n-1 \}. 
\]
Once, we set 
$\ol{H}_{i,j} = H_{i,j}\cap H_{m-1,n}\cap H_{m,n} \in (\ca{B}_{m,n}'')''$. Then, it is easy to show that 
the obtained arrangement 
$(\ca{B}'')'' = \{\ol{H}_{i,j} \subset \RR^{n-2} 
\mid i = 1, \dots, m-1, j = 1, \dots, n-1 \}$ satisfies the following  
relations: 
\begin{align*}
&\ol{H}_{i_1,j}\cap \ol{H}_{i_2,j}\cap \ol{H}_{i_3,j} = \emptyset, \\
&\ol{H}_{i_1,j_1}\cap \ol{H}_{i_1,j_2} \cap \ol{H}_{i_2,j_1} = 
 \ol{H}_{i_1,j_1}\cap \ol{H}_{i_1,j_2} \cap \ol{H}_{i_2,j_2} = 
 \ol{H}_{i_1,j_1}\cap \ol{H}_{i_2,j_1} \cap \ol{H}_{i_2,j_2}. 
\end{align*}
for $i_1, i_2, i_3 = 1, \dots, m-1$ and $j, j_1, j_2 = 1, \dots, n-1$.
This means that the hyperplane arrangement $(\ca{B}'')''$ can be 
regarded as an arrangement ``$\ca{B}_{m-1,n-1}$ in $\RR^{n-2}$''. 
We will subsequently justify this argument more precisely. 

Before we do it, we shall observe the deletion and restriction for 
$\ca{B}_{m,n}'$ with $H_{m-1, n} \in \ca{B}_{m,n}'$. 
Then, 
we have the following arrangements: 
\begin{align*}
 (\ca{B}_{m,n}')' &= 
\{ H_{1,1}, \dots, H_{m-2,n}, H_{m-1, 1}, \dots, H_{m-1,n-1}, 
\dots, H_{m,1}, \dots, H_{m,n-1} \}, \\
 (\ca{B}_{m,n}')'' &= 
 \left\{ H_{11}\cap H_{m-1,n}, \dots, H_{1,n}\cap H_{m-1,n}, \dots,  
 H_{m-1,1}\cap H_{m-1,n}, \dots, \atop H_{m-1,n-1}\cap H_{m-1,n}, 
  H_{m,1}\cap H_{m-1,n}, \dots, H_{m,n}\cap H_{m-1,n} \right\}.  
\end{align*}
Then, we remark that $(\ca{B}_{m,n}')''$ is same as 
$(\ca{B}_{m,n}'')'$ if we exchange $H_{m-1,j}$ and $H_{m,j}$. 
By these relations, we have the following diagram: 
\[
  \xymatrix@C=50pt@R=20pt{
    \ca{B}_{m,n} \ar[r]^{\cap H_{m,n}} \ar[d]^{{\setminus}H_{m,n}} & \ca{B}_{m,n}'' \ar[r]^{\hspace{-15pt}\cap (H_{m-1,n}\cap H_{m,n})} 
    \ar[d]^{{\setminus}(H_{m-1,n}\cap H_{m,n})}  &  (\ca{B}_{m,n}'')''\\
    \ca{B}_{m,n}' \ar[r]^{\hspace{-40pt}\cap H_{m-1,n}} 
    \ar[d]^{{\setminus}H_{m-1,n}} & 
    (\ca{B}_{m,n}')''\text{ ``}=\text{'' } (\ca{B}_{m,n}'')' & \\ 
     (\ca{B}_{m,n}')' 
  }
\]
To extract a recurrence relation from this diagram, we introduce 
another notation: Let 
\[ 
 \ca{B}^\ell_{m,n} = \{ X_{i,j}\subset \RR^\ell \mid i = 1, \dots, m, j = 1, \dots, n  \}
\] be a hyperplane arrangement in $\RR^\ell$ satisfying the following 
relations: 
\begin{align}
&X_{i_1,j}\cap X_{i_2,j}\cap X_{i_3,j} = \emptyset, \label{eq:condition-intersection-1-appendix}\\ 
&X_{i_1,j_1}\cap X_{i_1,j_2} \cap X_{i_2,j_1} = 
 X_{i_1,j_1}\cap X_{i_1,j_2} \cap X_{i_2,j_2} = 
 X_{i_1,j_1}\cap X_{i_2,j_1} \cap X_{i_2,j_2} 
 \label{eq:condition-intersection-2-appendix}
\end{align}
for $i_1, i_2, i_3 = 1, \dots, m$ and $j, j_1, j_2 = 1, \dots, n$. 
Then, by the above arguments and a simple consideration, 
we have the following diagram: 
\begin{gather}
\begin{aligned}
  \xymatrix@C=70pt@R=20pt{
    \ca{B}_{m,n}^n \ar[r]^{\cap X_{m,n}} \ar[d]^{{\setminus}X_{m,n}} & (\ca{B}_{m,n}^{n})'' \ar[r]^{\cap (X_{m-1,n}\cap X_{m,n})}
    \ar[d]^{{\setminus} (X_{m-1,n}\cap X_{m,n})}  & 
    \ca{B}_{m-1,n-1}^{n-2}\\
    (\ca{B}_{m,n}^n)' \ar[r]^{\cap X_{m-1,n}} \ar[d]^{{\setminus}X_{m-1,n}} & ((\ca{B}_{m,n}^n)')'' \ar[r]^{\cap (X_{m-2,n}\cap X_{m-1,n})}
    \ar[d]^{{\setminus} (X_{m-2,n}\cap X_{m-1,n})}
     & \ca{B}_{m-1,n-1}^{n-2} \\ 
     \vdots \ar[d]^{{\setminus}X_{2,n}} & \vdots 
     \ar[d]^{{\setminus} (X_{1,n}\cap X_{2,n})} \\  
     \ca{B} \ar[r]^{\cap X_{1,n}} \ar[d]^{{\setminus}X_{1,n}} & \ca{B}_{m,n-1}^{n-1} \\ 
     \ca{B}_{m,n-1}^{n}
  } 
\end{aligned}
\label{eq:commutative-diagram-deletion-restriction}
\end{gather}
Here, $\ca{B}$ is the hyperplane arrangement in $\RR^n$ 
defined by 
\[
 \ca{B} = \ca{B}_{m,n-1}^n \cup \{ X_{1, n} \}. 
\]
Let $b_{m,n}^\ell = |\Ch(\ca{B}_{m,n}^\ell)|$. Then, 
by Theorem~\ref{thm:deletion-restriction} with the 
diagram~\eqref{eq:commutative-diagram-deletion-restriction}, 
we have the recurrence relation 
\begin{align*}
 b_{m,n}^n = b_{m,n-1}^n + m b_{m,n-1}^{n-1} + 
 \frac{m(m-1)}{2}b_{m-1,n-1}^{n-2}. 
\end{align*}
Moreover, by considering recursively, we can show that 
the following holds for $\ell, m,n \geq 1$: 
\begin{align}
 b_{m,n}^\ell = b_{m,n-1}^{\ell} + m b_{m,n-1}^{\ell-1} + 
 \frac{m(m-1)}{2}b_{m-1,n-1}^{\ell-2}. \label{eq:recurrence-relation} 
\end{align} 
Here, $b^0_{m,n} = b_{0,n}^\ell = b_{m,0}^\ell = 1$ for any 
$\ell,m,n\geq 0$ and we set $b_{m,n}^\ell = 0$ for $\ell < 0$.  
Then, for example, by \eqref{eq:recurrence-relation}, we have 
 $b^1_{m,n} = mn + 1$ for any $m,n\geq 0$,  
$b^\ell_{m,1} = m^2/2 + m/2 + 1$ for any $\ell\geq 2$ and $m$. 
In particular, $b_{m,n}^\ell$ is a polynomial with respect to $m$. 
  
By this recurrence relation \eqref{eq:recurrence-relation}, we can represent $b_{m,n}^n$ as 
\[
 b_{m,n}^n = \sum_{k=0}^{n/2} \sum_{\ell=0}^{n} 
 d_{\ell, k}(m) b_{m-k,0}^{n-2k-\ell} = \sum_{k=0}^{n/2} \sum_{\ell=0}^{n} 
 d_{\ell, k}(m), 
\]
where $d_{\ell,k}(m)$ is a non-negative integer. Here, the last equation follows from $b_{m-k,0}^{n-2k-\ell} = 1$ for any 
$k, \ell, m$ such that $m-k\geq 0$ and $n-2k-\ell \geq 0$.    
Then, it is easy to show that $d_{\ell, k}(m)$ is obtained 
as a sum of multiples of $k$ times 
``$m(m-1)/2$'', $\ell$ times ``$m$'', and $n-k-\ell$ times $1$. 
Here, these double quotation means that these vary in accordance with  
 the order of the operations. Indeed, the iteration relation~\eqref{eq:recurrence-relation} can be represented as 
 a higher-dimensional analogue of Pascal's triangle as Figure~\ref{fig:Pascal's-triangle}. 
\begin{figure}[h]
\centering
    \includegraphics[keepaspectratio,scale=0.5]{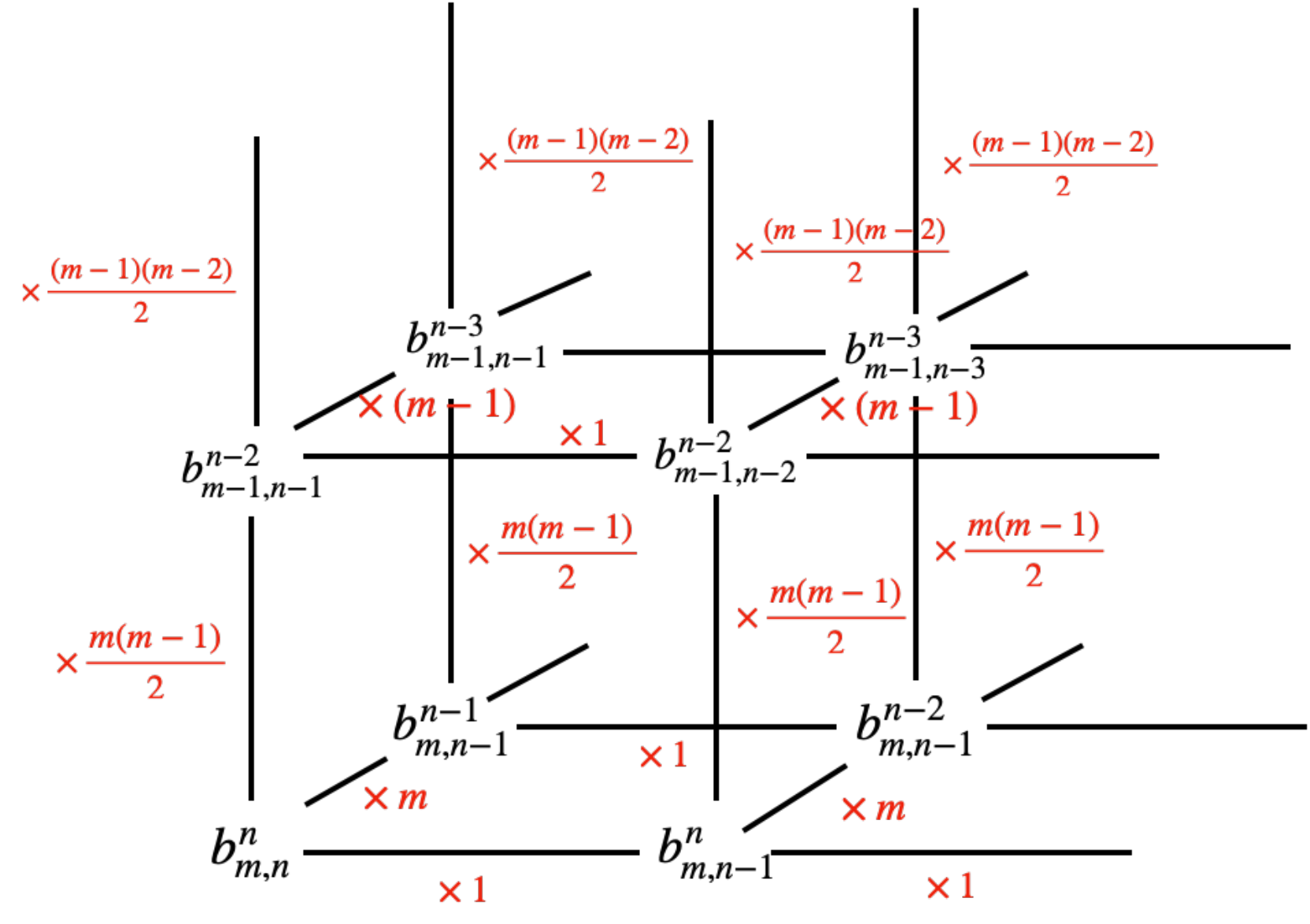}
    \caption{A higher dimensional analogue of Pascal's triangle representing the iteration relation~\eqref{eq:recurrence-relation}.}
    \label{fig:Pascal's-triangle}
\end{figure}
However, because we will calculate only the coefficient of leading term of $b_{m,n}^n$ as a polynomial of variable $m$, we may not 
take care of the orders.  
Then, the degree of $d_{\ell, k}(m)$ as a polynomial of variable $m$ is equal to $2k + \ell$. 
This means that the leading term of 
$b_{m,n}^n$ as a polynomial of variable $m$ is equal to the 
sum of terms $d_{\ell, k}(m)$ for $2k+\ell = n$. 
Moreover, by the fact $d_{\ell, k}(m)\geq 0$, we have 
\begin{align}
 b_{m,n}^n = \sum_{k=0}^{n/2} \sum_{\ell=0}^{n} 
 d_{\ell, k}(m)  
 \geq \sum_{k=0}^{n/2} d_{n-2k, k}(m) = 
 \text{(the leading term of }b_{m,n}^n \text{ as a polynomial of variable } m \text{)}.  
 \label{eq:lower-bound-by-recurrence}
\end{align}
We calculate a lower bound of the leading term. 
Then, the leading term of $d_{n-2k, k}(m)$ as a polynomial of $m$ can be written as   
\begin{align}
 d_{n-2k,k}(m) = {n \choose k, k, n-2k} \frac{1}{2^{k}} 
 m^n + O(m^{n-1}), \label{eq:first-lower-bound-of-d}
\end{align}
where ${n \choose k_1, \dots, k_m}$ for positive integers 
$k_1, \dots, k_m$ such that $n = k_1 + \cdots + k_m$ 
is the multinomial coefficient defined by 
\begin{align}
 {n \choose k_1, \dots, k_m} = \frac{n!}{k_1! \cdots k_m!} = 
 {k_1 \choose k_1}{k_1+k_2 \choose k_2}\cdots {k_1+k_2+\cdots + k_m \choose k_m}. \label{eq:multinomial-coefficient} 
\end{align}
Indeed, as mentioned before, $d_{n-2k,k}(m)$ is obtained as a sum of multiples of $k$ times of ``$m(m-1)/2$'', $n-2k$ times of ``$m$'', and $k$ times of $1$. Although the terms in the double quotations varies in accordance with the orders of the operations, the leading term 
is independent of the orders. Hence, the leading term of 
$d_{n-2k,k}(m)$ 
is the sum of multiples of $k$ times of $1/2$, $n-2k$ times of $1$, 
 and $k$ times of $1$. The number of such multiples in the sum is same 
 as ${n \choose k, k, n-2k}$. Hence, we have 
\[
 d_{n-2k, k}(m) = {n \choose k, k, n-2k} \frac{1}{2^{k}} 
 m^n + O(m^{n-1}).  
\]

By the form of RHS of equation~\eqref{eq:multinomial-coefficient} 
and the estimate in \eqref{eq:estimate-linear-region-fully-connected}, 
we have 
\begin{align}
 {n \choose k, k, n- 2k} &= {k \choose k}
 {2k \choose k}{n \choose n- 2k} = 
 {2k \choose k}{n \choose n- 2k}\notag \\
 &\geq \frac{2^{k H(1/2)}}{\sqrt{8k(1-1/2)}}
 \frac{2^{nH((n-2k)/n)}}{\sqrt{8k(n-2k)(1-(n-2k)/n)}} \notag \\ 
 &= \frac{2^{2k}2^{nH((n-2k)/n)}}{8k\sqrt{(n-2k)/n}}. 
 \label{eq:lower-bound-of-multinomial-coeff}
\end{align}
In the last inequality follows from $H(1/2) = 1$. 

We evaluate the coefficient of the leading term at $k= n/4$. Then,  
we have
\[
 d_{n/2, n/4}(m) \geq \frac{(2^{5/4})^n}{n\sqrt{2}} 
 m^n + O(m^{n-1}). 
\]
In particular, the coefficient of leading term of $b_{m,n}^n$ is 
bounded from below by $(2^{5/4})^n/(n\sqrt{2})$. 
This concludes the proof. \qedhere

\end{proof}

\section{Proof of Proposition~\ref{prop:characterization-group-invariant}}
\label{sec:proof-characterization-group-invariant}

\begin{proof}[Proof of Proposition~\ref{prop:characterization-group-invariant}]
Let $\lambda\in \Lambda$, $\bm{x}\in D_\lambda$ and $\phi \in \Phi$. We assume that $\phi$ satisfies 
(1) $\phi(D_\lambda) = D_{\lambda'}$ and 
(2) $f_\lambda = f_\lambda\circ \phi|_{D_\lambda}$. 
Then, we have 
\begin{align}
 f(\phi(\bm{x})) &= f_{\lambda'}(\phi(\bm{x})) = 
 (f_{\lambda'}\circ \phi|_{D_\lambda}) (\bm{x}) \\ &= 
 f_\lambda(\bm{x}) = f(\bm{x}). \label{eq:invariance-of-phi}
\end{align}
This equation holds for any $\bm{x}$ and any $\phi \in \Phi$. 
Because $\phi \in \Phi$ is a Euclidean transformation, $\phi$ is an isomorphism. In particular, the inverse of $\phi$ exists. 
As for any $\bm{y} \in \RR^n$, there is a $\bm{x}$ such that 
$\bm{y} = \phi(\bm{x})$, by the equation~\eqref{eq:invariance-of-phi}, 
we have 
\begin{align}
 f(\phi^{-1}(\bm{y})) = f(\bm{x}) = f(\phi(\bm{x})) = f(\bm{y}). 
 \label{eq:invariance-of-phi^{-1}}
\end{align}
Hence, $f$ is invariant by the action of $\phi^{-1}$ for any 
$\phi\in \Phi$.  
Now, let $\wh{\Phi}$ be the subgroup of the group of Euclidean transformations generated by $\Phi$. This means that any element 
$\phi \in \wh{\Phi}$ is a composition of finite elements of 
$\{\phi_1, \dots, \phi_t, \phi_1^{-1}, \dots, \phi_t^{-1} \}$. 
Hence, by combining this fact and equations \eqref{eq:invariance-of-phi} and \eqref{eq:invariance-of-phi^{-1}}, 
$f$ is invariant by the action of the group $\wh{\Phi}$. 	
\end{proof}

\end{document}